\documentclass[sigconf]{aamas}

\usepackage{algorithm2e}
\usepackage{balance} 
\usepackage{bbm}
\usepackage{enumerate}
\usepackage{listings}
\usepackage{multirow}
\usepackage{wrapfig}

\RestyleAlgo{ruled}

\makeatletter
\gdef\@copyrightpermission{
  \begin{minipage}{0.2\columnwidth}
   \href{https://creativecommons.org/licenses/by/4.0/}{\includegraphics[width=0.90\textwidth]{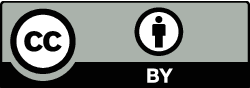}}
  \end{minipage}\hfill
  \begin{minipage}{0.8\columnwidth}
   \href{https://creativecommons.org/licenses/by/4.0/}{This work is licensed under a Creative Commons Attribution International 4.0 License.}
  \end{minipage}
  \vspace{5pt}
}
\makeatother

\setcopyright{ifaamas}
\acmConference[AAMAS '25]{Proc.\@ of the 24th International Conference
on Autonomous Agents and Multiagent Systems (AAMAS 2025)}{May 19 -- 23, 2025}
{Detroit, Michigan, USA}{Y.~Vorobeychik, S.~Das, A.~Nowé  (eds.)}
\copyrightyear{2025}
\acmYear{2025}
\acmDOI{}
\acmPrice{}
\acmISBN{}

\acmSubmissionID{944}

\title{Neural DNF-MT: A Neuro-symbolic Approach for\\Learning Interpretable and Editable Policies}

\author{Kexin Gu Baugh}
\affiliation{
  \institution{Imperial College London}
  \city{London}
  \country{United Kingdom}}
\email{kexin.gu17@imperial.ac.uk}

\author{Luke Dickens}
\affiliation{
  \institution{University College London}
  \city{London}
  \country{United Kingdom}}
\email{l.dickens@ucl.ac.uk}

\author{Alessandra Russo}
\affiliation{
  \institution{Imperial College London}
  \city{London}
  \country{United Kingdom}}
\email{a.russo@imperial.ac.uk}
\authornote{Sponsored in part by DEVCOM Army Research Lab under W911NF2220243.}

\begin{abstract}

Although deep reinforcement learning has been shown to be effective, the model's
black-box nature presents barriers to direct policy interpretation. To address
this problem, we propose a neuro-symbolic approach called neural DNF-MT for
end-to-end policy learning. The differentiable nature of the neural DNF-MT model
enables the use of deep actor-critic algorithms for training. At the same time,
its architecture is designed so that trained models can be directly translated
into interpretable policies expressed as standard (bivalent or probabilistic)
logic programs. Moreover, additional layers can be included to extract abstract
features from complex observations, acting as a form of predicate invention. The
logic representations are highly interpretable, and we show how the bivalent
representations of deterministic policies can be edited and incorporated back
into a neural model, facilitating manual intervention and adaptation of learned
policies. We evaluate our approach on a range of tasks requiring learning
deterministic or stochastic behaviours from various forms of observations. Our
empirical results show that our neural DNF-MT model performs at the level of
competing black-box methods whilst providing interpretable policies.

\end{abstract}

\keywords{Neuro-symbolic Learning; Neuro-symobilc Reinforcement Learning}

\newcommand{\BibTeX}{\rm B\kern-.05em{\sc i\kern-.025em b}\kern-.08em\TeX}

\newcommand{\indicatorf}{\mathbbm{1}}
\newcommand{\indicatorfof}[1]{\indicatorf_{#1}}
\newcommand{\mutextanh}{\text{mutex-tanh}}
\newcommand{\softmax}{\text{softmax}}

\newtheorem{definition}{Definition}[section]
\theoremstyle{definition}\newtheorem{remark}{Remark}

\begin{document}


\pagestyle{fancy}
\fancyhead{}


\maketitle


\section{Introduction}


Remarkable progress has been made in reinforcement learning (RL) with the
advancement of deep neural networks. Since the demonstration of impressive
performance in complex games like Go \cite{alphago} and Dota 2
\cite{open-ai-five}, significant effort has been made to utilise deep RL
approaches for solving real-life problems, such as segmenting surgical gestures
\cite{surgical-gesture} and providing treatment decisions \cite{q-value-sepsis}.
However, the need for model interpretability grows with safety and ethical
considerations. In the EU's AI Act, systems used in areas such as healthcare
fall into the high-risk category, requiring both a high level of accuracy and a
method to explain and interpret their output\cite{ai-act}. Therefore, the
`black-box' nature of neural models becomes a concern when using them for such
high-stakes decisions in healthcare \cite{model-interpretability-healthcare}.
While many approaches exist to \textit{explain} black-box neural models with
post-hoc methods, it is argued that using inherently \textit{interpretable}
models is safer \cite{interpretable-model}.


Various neuro-symbolic approaches address the lack of interpretability in deep
RL. We use the term `symbolic' to refer to methods that offer \textit{logical
rule representations}, in contrast to program synthesis approaches \cite{pirl,
propel, galois} that offer \textit{programmatic representations} with forms of
logic. Some of these neuro-symbolic methods \cite{nlrl, nudge} rely on manually
engineered inductive bias to restrict the search space and thus limit the rules
they can learn. Others \cite{nsrl-fol, dlm} without predefined inductive bias
associate weights with predicates but require pre-trained components to parse
observations to predicates \cite{nsrl-fol} or a special critic for training
\cite{dlm}.


In this paper, we propose a neuro-symbolic model, neural DNF-MT, for
learning interpretable and editable policies.%
\footnote{Our main experiment repo is available at
      \url{https://github.com/kittykg/neural-dnf-mt-policy-learning}.} %
Our model is built upon the semi-symbolic layer and neural DNF model proposed in
pix2rule \cite{pix2rule} but with modifications that support probabilistic
representation for policy learning. The model is completely differentiable and
supports integration with deep actor-critic algorithms. It can also be used to
distil policies from other neural models. From trained neural DNF-MT actors, we
can extract bivalent logic programs for deterministic policies or probabilistic
logic programs for stochastic policies. These interpretable logical
representations are close approximations of the learned models. The
neural-bivalent-logic translation is bidirectional, thus enabling manual policy
intervention on the model. We can modify the bivalent logical program and port
it back to the neural model, benefiting from the tensor operations and
environment parallelism for fast inference. Compared to existing works, we do
not rely on rule templates or mode declarations. Furthermore, our model is
trained with a simple MLP critic and supports trainable preceding layers to
generalise relevant facts from complex observations, such as multi-dimensional
matrices.

To summarise, our main contributions are:

\begin{enumerate}
      \item We propose neural DNF-MT, a neuro-symbolic model for end-to-end
            policy learning and distillation, without requiring manually
            engineered inductive bias. It can be trained with deep actor-critic
            algorithms and supports end-to-end predicate invention.

      \item A trained neural DNF-MT actor's policy can be represented as a logic
            program (probabilistic for a stochastic policy and bivalent for a
            deterministic policy), thus providing interpretability.

      \item The neural-to-bivalent-logic translation is bidirectional, and we
            can modify the logical program for policy intervention and port it
            back to the neural model, benefiting from tensor operations and
            environment parallelism for fast inference.
\end{enumerate}

\section{Background}

\subsection{Reinforcement Learning}

RL tasks are commonly modelled as Markov Decision Processes (MDPs) \cite{mdp} or
sometimes Partially Observable Markov Decision Processes (POMDPs) \cite{pomdp-1,
    pomdp-2}, depending on whether the observed states are fully Markovian. The
objective of an RL agent is to learn a policy that maps states to action
probabilities $\pi(a_t|s_t)$ to maximise the cumulative reward. Value-based
methods such as Q-learning \cite{q-learning} and Deep Q-Networks (DQN)
\cite{dqn} approximate the action-value function $Q(s_t, a_t)$, while
policy-based methods such as REINFORCE \cite{reinforce} directly parameterise
the policy $\pi$. Actor-critic algorithms such as Advantage Actor-Critic (A2C)
\cite{a3c} and Proximal Policy Optimisation (PPO) \cite{ppo} combine both
value-based and policy-based methods, where the actor learns the policy
$\pi(a_t|s_t)$ and the critic learns the value function $V(s_t)$. Specifically,
PPO clips the policy update in a certain range to prevent problematic large
policy changes, providing stability and better performance.

\subsection{Semi-symbolic Layer and Neural DNF Model}

A neural Disjunctive Normal Form model \cite{pix2rule} is a fully differentiable
neural architecture where each node can be set to behave like a semi-symbolic
conjunction or disjunction of its inputs. For some trainable weights $w_i$, $i =
    1, \ldots, I$, and a parameter $\delta$, a node in the neural DNF model is given
by:
\begin{align}
    \hat{y} & = \tanh\left(\sum_{i=1}^{I} w_i x_i + \beta\right),
    \; \text{with} \;
    \beta   = \delta \left(\max_{i=1}^{I}|w_i| - \sum_{i=1}^{I} |w_i|\right) \label{eq:ss-bias}
\end{align}
Here the $I$ (semi-symbolic) inputs to the node are constrained such that $x_i
    \in [-1, 1]$, where the extreme value $1$ ($-1$) is interpreted as
associated term $i$ taking the logical value $\top$ ($\bot$) with other
values representing intermediate strengths of belief (a form of fuzzy logic
or generalised belief). The node activation $\hat{y} \in (-1, 1)$ is interpreted
similarly but cannot take specific values $1$ or $-1$. The node's
characteristics are controlled by a hyperparameter $\delta$, which induces
behaviour analogous to a logical conjunction (disjunction) when $\delta = 1$
($= -1$). The neural DNF model consists of a layer of conjunctive nodes
followed by a layer of disjunctive nodes. During training, the absolute
value of each $\delta$ in both layers is controlled by a scheduler that
increases from 0.1 to 1, as the model may fail to learn any rules if the
logical bias is at full strength at the beginning of training.

Pix2rule \cite{pix2rule} proposes interpreting trained neural DNF models as
logical rules with Answer Set Programming (ASP) \cite{asp} semantics by treating
each node's output $\hat{y} > 0$ ($\le 0$) as logical $\top$ ($\bot$) (akin to a
maximum likelihood estimate of the associated fact). However,
\citet{ns-classifications} point out that the neural DNF models cannot be used
to describe multi-class classification problems because the disjunctive layer
fails to guarantee a logically mutually exclusive output, i.e. with exactly one
node taking value $\top$. \citet{ns-classifications} instead propose an extended
model called neural DNF-EO, which adds a non-trainable conjunctive semi-symbolic
layer after the final layer of the base neural DNF to approximate the
`exactly-one' logical constraint `$class_j \leftarrow \land_{k, j \neq k} \text{
        not } class_k$', and again show how ASP rules can be extracted from trained
models.

\section{Neural DNF-MT Model}

This section explains why existing neural DNF-based models from \cite{pix2rule}
and \cite{ns-classifications} are imperfectly suited to represent policies
within a deep-RL agent, and presents a new model called neural DNF with
mutex-tanh activation (neural DNF-MT) to address these limitations. It then
shows how trained models can be variously interpreted as deterministic and
stochastic policies for the associated domains.

\subsection{Issues of Existing Neural DNF-based Models}\label{sec:existing-model-issue}

Unlike multi-class classification, where each sample has a single deterministic
class, an RL actor seeks to approximate the optimal policy with potentially
arbitrary action probabilities \cite{rl-book-sutton-barto}. It is possible for a
domain to have an optimal deterministic policy and for the RL algorithm to
approach it with an `almost deterministic' policy, where for each state the
optimal action's probability is significantly greater than the others (i.e. a
single almost-1 value vs all the rest close to 0). In this case, the actor
almost always chooses a single action, similar to a multi-class classification
model predicting a single class. A trained neural DNF-based model representing
such a policy should be interpreted as a bivalent logic program representing the
nearest deterministic policy. When we wish to preserve the probabilities encoded
within the trained neural DNF-based actor without approximating it with the
nearest deterministic policy, its interpretation should be captured as a
probabilistic logic program that expresses the action distributions. There is no
way to achieve both of these objectives with the neural DNF and neural DNF-EO
models, since their interpretation frameworks do not satisfy two forms of mutual
exclusivity: (a) \emph{probabilistic mutual exclusivity} when interpreted as a
stochastic policy, and (b) \emph{logical mutual exclusivity} when interpreted as
a deterministic policy. We first formalise the logic system represented by
neural DNF-based models (Definition~\ref{def:gbl}) and then define the two
mutual exclusivities possible in this logic system (Definition~\ref{def:tanh-me}
and \ref{def:prob-me}).

\begin{definition}[Generalised Belief Logic] \label{def:gbl}

    A neural DNF-based model that builds upon semi-symbolic layers represents a
    logic system. We refer to this logic system as \textbf{Generalised Belief
        Logic (GBL)}. A semi-symbolic node's activation $y_i \in (-1, 1)$ represents
    its belief in a logical proposition. For each activation $y_i$, we define a
    bivalent logic variable $b_i \in \{\bot, \top\}$ as its bivalent logical
    interpretation:
    \[
        b_i = \begin{cases}
            \top & \text{if } y_i > 0 \\
            \bot & \text{otherwise}
        \end{cases}
    \]

\end{definition}

\begin{definition}[Logical mutual exclusivity] \label{def:tanh-me}

    Given the final activation of a neural DNF-based model for $N$ classes
    $\mathbf{y} \in (-1, 1)^{N}$ and its bivalent logic interpretation
    $\mathbf{b} \in \{\bot, \top\}^{N}$, the model satisfies \textbf{logical
        mutual exclusivity} if there is exactly one $b_i$ that is $\top$:
    \[
        \models \left( \bigvee_{i \in \{1..N\}} b_i \right) \land
        \left( \bigwedge_{i, j \in \{1..N\},i < j} \neg (b_i \land b_j) \right)
    \]

\end{definition}

\begin{definition}[Probabilistic mutual exclusivity] \label{def:prob-me}

    A probabilistic interpretation of GBL is a function $f_p: (-1, 1) \to (0,
        1)$ that maps each belief $y_i$ to probability $p_i$ that $b_i$ holds as
    true. Formally,
    \[
        p_i = f_p(y_i) = \Pr(b_i = \top | y_i)
    \]
    A neural DNF-based model satisfies \textbf{probabilistic mutual exclusivity}
    if the interpreted probabilities associated with its activations $\mathbf{y}
        \in (0, 1)^{N}$ under probabilistic interpretation $f_p$ sum to 1. That is:
    \[
        \sum_i^{N} f_p(y_i) = 1
    \]

\end{definition}

To be used for interpretable policy learning, a neural DNF-based model must
guarantee the following properties:

\begin{enumerate}[P1:]
    \item\label{desire-property-problog} The model provides a probabilistic
          mutually exclusive interpretation (Definition~\ref{def:prob-me}) and
          can be interpreted as a probabilistic logic program (such as ProbLog
          \cite{problog}),

    \item\label{desire-property-asp} When the \textit{optimal policy is
              deterministic}, the model can also be interpreted as a bivalent
          logic program (such as ASP \cite{asp}) that satisfies logical
          mutual exclusivity (Definition~\ref{def:tanh-me}).
\end{enumerate}

A trained neural DNF model from \cite{pix2rule} does not provide probabilistic
interpretation or guarantee logical mutual exclusivity in its bivalent
interpretation, and thus fails \hyperref[desire-property-problog]{P1} and
\hyperref[desire-property-asp]{P2}. A trained neural DNF-EO from
\cite{ns-classifications} satisfies \hyperref[desire-property-asp]{P2} via its
constraint layer but fails to provide probabilistic interpretation for
\hyperref[desire-property-problog]{P1}. To address these requirements, we
propose a new model called neural DNF-MT and post-training processing steps that
translate a trained neural DNF-MT model into a ProbLog program and, where
applicable, into an ASP program. Our proposed model satisfies both properties
above.

\subsection{Mutex-tanh Activation}\label{sec:ndnf-mt-method}

Let $\mathbf{d} \in \mathbb{R}^{N}$ be the output vector of a disjunctive
semi-symbolic layer before any activation function and $d_k \in \mathbb{R}$ be
the output of the $k$\textsuperscript{th} disjunctive node. Using the softmax
function, we define the new activation function $\mutextanh$ as:
\begin{align}
    \softmax(\mathbf{d})_k   & = \frac{e^{d_{k}}}{\sum_{i}^{N} e^{d_{i}}} \notag          \\
    \mutextanh(\mathbf{d})_k & = 2 \cdot \softmax(\mathbf{d})_k - 1 \label{eq:mutex-tanh}
\end{align}

With the $\mutextanh$ activation function, our neural DNF-MT model is
constructed with a semi-symbolic conjunctive layer with a $\tanh$ activation
function and a disjunctive semi-symbolic layer with the $\mutextanh$ activation
function:
\begin{align*} \label{eq:ndnf-mt-full-architecture}
    \mathbf{c}         & = \tanh(\mathbf{W_c} \mathbf{x} + \mathbf{\beta_c}) &  & \text{Output of conj. layer}            \\
    \mathbf{d}         & = \mathbf{W_d} \mathbf{c} + \mathbf{\beta_d}        &  & \text{Raw output of disj. layer}        \\
    \mathbf{\tilde{y}} & = \mutextanh(\mathbf{d})                            &  & \text{mutex-tanh output of disj. layer}
\end{align*}
where $\mathbf{W_c}$ and $\mathbf{W_d}$ are trainable weights, and
$\mathbf{\beta_c}$ and $\mathbf{\beta_d}$ are the logical biases calculated as
Eq~(\ref{eq:ss-bias}). Note that $\mathbf{\tilde{y}} \in (-1, 1)^{N}$ shares the
same codomain as the disjunctive layer's $\tanh$ output $\hat{\mathbf{y}}$. The
disjunctive layer's bivalent interpretation $\hat{\mathbf{b}}$ still uses
$\hat{\mathbf{y}}$, with $\hat{b_i} = \top$ when $\hat{y_i} > 0$ and $\bot$
otherwise.

To satisfy \hyperref[desire-property-problog]{P1}, we compute the probability
$\mathbf{\tilde{p}}$ as:
\begin{equation} \label{eq:ndnf-mt-prob}
    \mathbf{\tilde{p}} = \left(f_p(\tilde{y}_1), \ldots, f_p(\tilde{y}_N)\right)^T
    \quad \text{where} \quad f_p(\tilde{y}_i) = \frac{\tilde{y}_i+1}{2}
\end{equation}
By construction, $\mathbf{\tilde{p}} \in (0, 1)^N$, and we have $\sum_k^{N}
    \tilde{p_k} = 1$ from Eq~(\ref{eq:mutex-tanh}) to satisfy probabilistic
mutual exclusivity.

\subsection{Policy Learning with Neural DNF-MT}\label{sec:policy-learning-ndnf-mt}

In the following, we show how the neural DNF-MT model can be trained in an
end-to-end fashion to approximate a stochastic policy and how to extract the
policy into interpretable logical form.


\noindent \textbf{Training Neural DNF-MT as Actor with PPO.}\ \ Using the PPO
algorithm \cite{ppo}, we train a neural DNF-MT actor with an MLP critic. The
input to the neural DNF-MT actor must be in $[-1, 1]^I$. Any discrete
observation is converted into a bivalent vector representation, as shown in
Figure \ref{fig:binary-ndnf-mt-ac}. If the observation is complex, as shown in
our experiment in Section~\ref{sec:door-corridor}, an encoder can be added
before the neural DNF-MT actor to invent predicates in GBL form. The encoder
output acts as input to the neural DNF-MT actor and the MLP critic, as shown in
Figure \ref{fig:image-ndnf-mt-ac}.

\begin{figure*}[t]
    \begin{minipage}{0.44\textwidth}
        \centering
        \includegraphics[width=\linewidth]{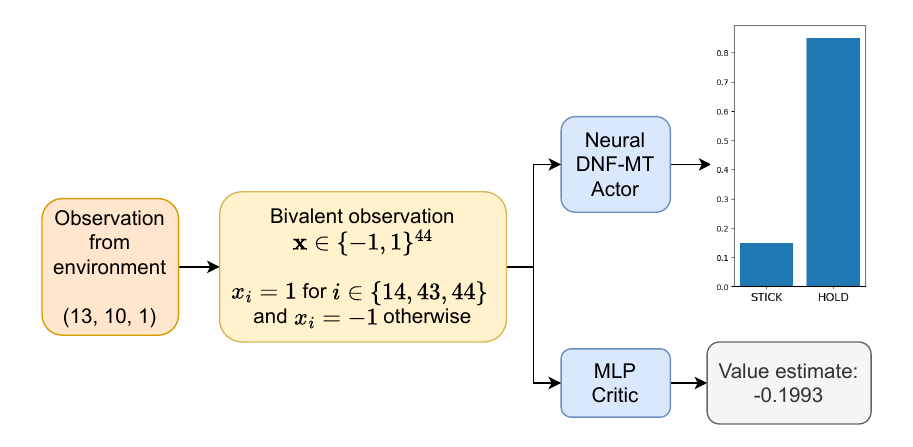}
        \caption{Neural DNF-MT model as an actor in actor-critic PPO, in
            environments with discrete observations.}
        \label{fig:binary-ndnf-mt-ac}
        \Description{Neural DNF-MT model in PPO algorithm for environments with
            categorical/factual observation.}
    \end{minipage} \hspace{0.5cm} %
    \begin{minipage}{0.48\textwidth}
        \centering
        \includegraphics[width=\linewidth]{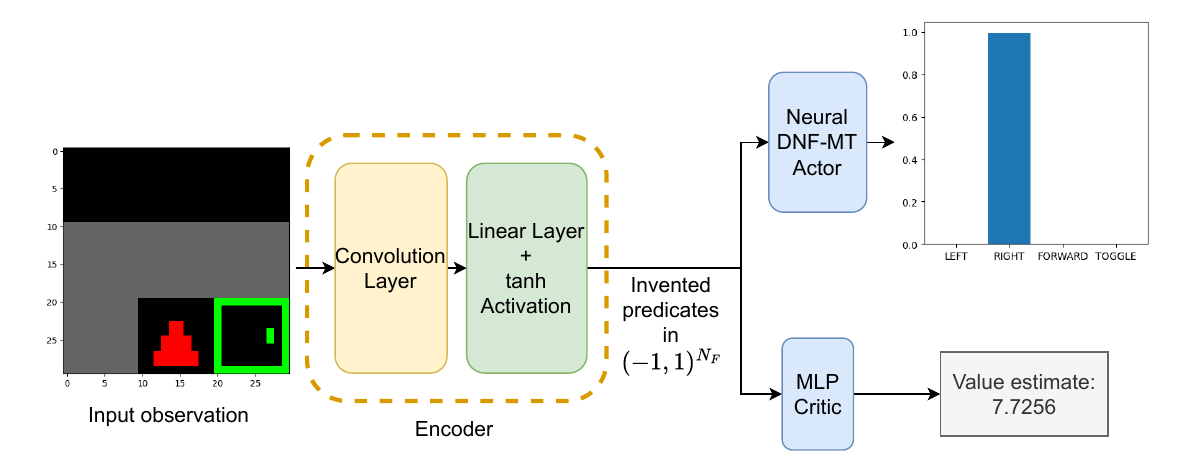}
        \caption{Neural DNF-MT model as an actor in actor-critic PPO, in
            environments with complex observations, such as an image-like
            multi-dimensional matrix.}
        \label{fig:image-ndnf-mt-ac}
        \Description{Neural DNF-MT model in PPO algorithm for environments with
            complex observation.}
    \end{minipage}
\end{figure*}

\begin{figure*}[!ht]
    \centering
    \includegraphics[width=0.92\linewidth]{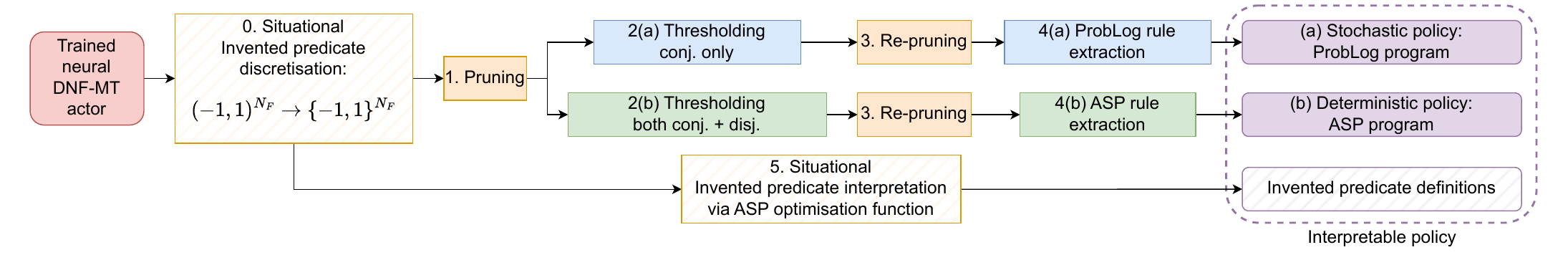}
    \caption{Post-training processing to extract an interpretable logical policy
        from a trained neural DNF-MT actor. There are two branches: one with
        sub-label (a) for extracting a stochastic policy in ProbLog and the
        other with sub-label (b) for extracting a deterministic policy in ASP.}
    \label{fig:post-training-process}
    \Description{Post-training pipeline for a trained neural DNF-MT actor.}
\end{figure*}

We here present the overall training loss of the actor-critic PPO with a neural
DNF-MT actor, which consists of multiple loss terms. The base training loss
component matches that from PPO \cite{ppo}:
\begin{equation}
    L^{\text{PPO}}(\theta) = \mathbb{E}_t \left[ L^{\text{CLIP}}(\theta) +
        c_1 L^{\text{value}}(\theta) - c_2 S[\pi_\theta](s_t) \right]
\end{equation}
where $c_1, c_2 \in \mathbb{R}$ are hyperparameters, $L^{\text{CLIP}}(\theta)$
is the clipped surrogate objective, $S[\pi_\theta](s_t)$ is the entropy of the
actor in training, and $L^{\text{value}}(\theta)$ is the value loss. The
detailed explanations for each term are in Appendix~\ref{appendix:ppo-loss}. The
action probability output of the neural DNF-MT actor defined in
Eq~(\ref{eq:ndnf-mt-prob}) is used to calculate the probability ratio in
$L^{\text{CLIP}}$ and the entropy term $S[\pi_\theta](s_t)$.

We add the following auxiliary losses to facilitate the interpretation of the
neural DNF-MT model into rules:
\begin{align}
    L^{(1)}(\theta)
     & = \frac{1}{N_{F}} \sum_{i}^{N_{F}} \left\vert 1 - \left\vert f_i
    \right\vert \right\vert
    \label{eq:l-feature-enc-reg}
    \\
    L^{(2)}(\theta)
     & = \frac{1}{|\theta_{\text{disj}}|}\sum \left| \theta_{\text{disj}}
    \cdot (6 - |\theta_{\text{disj}}|) \right|
    \label{eq:l-disj-weight-reg}
    \\
    L^{(3)}(\theta)
     & = \frac{1}{N_{C}}\sum_{i}^{N_{C}} \left\vert 1 - \left\vert
    c_i \right\vert \right\vert
    \label{eq:l-conj-tanh-out-reg}
    \\
    L^{(4)}(\theta)
     & = - \sum_{i}^{N} \left[p_i \log \left(\frac{\hat{y}_i + 1}{2}\right) +
        (1 - p_i) \log \left(1 - \frac{\hat{y}_i + 1}{2}\right) \right]
    \label{eq:l-mutex-tanh}
\end{align}
where $f_i$ is the invented predicate, $N_{F}$ is the number of output of an
encoder, and $N_{C}$ is the number of conjunctive nodes.
Eq~(\ref{eq:l-feature-enc-reg}) is used when there is an encoder before the
neural DNF-MT actor for predicate invention. It enforces the predicates'
activations to be close to $\pm 1$ so that they are stronger beliefs of
true/false. Eq~(\ref{eq:l-disj-weight-reg}) is a weight regulariser to encourage
the disjunctive weights to be close to $\pm 6$ (the choice of $\pm 6$ is to
saturate $\tanh$, as $\tanh(\pm 6) \approx \pm 1$).
Eq~(\ref{eq:l-conj-tanh-out-reg}) encourages the $\tanh$ output of the
conjunctive layer to be close to $\pm 1$. Eq~(\ref{eq:l-mutex-tanh}) is the key
term to satisfy \hyperref[desire-property-asp]{P2}, pushing for bivalent logical
interpretations for deterministic policies. This term mimics a cross-entropy
loss between each $\mutextanh$ output and corresponding individual $\tanh$
outputs of the disjunctive layer, pushing the probability interpretations of the
$\tanh$ outputs (i.e. $(\hat{y}_i + 1)/2$) towards their action probability
$\tilde{p}_i$ counterparts. If the optimal policy is deterministic, all
$\tilde{p_i}$ will be approximately 0 except for one, which is close to 1. Each
$\hat{y_i}$ is pushed towards $\pm 1$, and only one will be close to 1, thus
having exactly one bivalent interpretation $b_i = \top$ and satisfying logical
mutual exclusivity.

Finally, the overall training loss is defined as:
\begin{equation}
    L(\theta) = L^{\text{PPO}}(\theta) + \sum_{i\in\{1, 2, 3, 4\}}
    \lambda_i L^{(i)}(\theta)
\end{equation}
where $\lambda_i \in \mathbb{R}, i \in \{1,2,3,4\}$ are hyperparameters.


\noindent \textbf{Post-training Processing.} \ \ This extracts either a ProbLog
program for a stochastic policy or an ASP program for a close-to-deterministic
policy from a trained neural DNF-MT actor, where the logic program is a close
approximation of the model. It consists of multiple stages, as shown in Figure
\ref{fig:post-training-process}, described as follows.


\textbf{(1) Pruning}: This step repeatedly passes over each edge that connects
an input to a conjunction or a conjunction to a disjunction, and removes any
edge that can be removed (i) without changing the learned trajectory (for
deterministic domains) or (ii) without shifting any action probability for any
state more than some threshold $\tau_{\text{prune}}$ from the original learned
policy (for stochastic domains). Any unconnected nodes are also removed. The
process terminates when a pass fails to remove any edges or nodes.

\label{sec:thresholding}\textbf{(2) Thresholding}: This process converts a
semi-symbolic layer's weights from $\mathbb{R}$ to values in $\{-6, 0, 6\}$.
Given some threshold $\tau \in \mathbb{R}_{\ge 0}$, a new weight is computed as
$w'_{\text{k}ij} = 6 \cdot \indicatorfof{|w_{\text{k}ij}| \ge
    \tau}(w_{\text{k}ij}) \cdot \text{sign}(w_{\text{k}ij}), \text{k} \in \{
    \text{c}, \text{d}\}$. This weight update enables the \textit{neural to bivalent
    logic translation} described \hyperref[sec:nbl-translation]{later}. The
selection of $\tau$ should maintain the model's trajectory/action probability,
subject to the same checks used in pruning. For a thresholded node with at least
one non-zero weight, we replace its $\tanh$ activation with step function $h(x)
    = 2 \cdot \indicatorfof{x > 0}(x) - 1$, changing its output's range to $\{-1,
    1\}$. The thresholding process is applied differently to the disjunctive layer
depending on the nature of the policy desired.

\begin{enumerate}[(2.a)]
    \item \textbf{For stochastic policies}: Only the conjunctive layer is
          thresholded, i.e. choosing a value of $\tau$, updating only its
          weights and changing the activation function. The disjunctive layer
          still outputs action probabilities.
    \item \textbf{For deterministic policies}: Thresholding is applied to both
          the conjunctive and disjunctive layers: a single value $\tau$ is
          chosen and applied in both layers' weight update, and both layers have
          their $\tanh$ activation replaced with the step function. This process
          is only possible if the model satisfies
          \hyperref[desire-property-asp]{P2}.
\end{enumerate}

\textbf{(3) Re-pruning}: The pruning process from Step 1 is repeated.

\textbf{(4) Logical rules extraction}: All nodes (conjunctive and disjunctive)
are converted into some form of logical rules. The thresholding process
guarantees that all conjunctive nodes can be translated into bivalent logic
representations. For a conjunctive node $c_j$, we consider the set
$\mathcal{X}_j = \{i \in \{1..I\} | w'_{\text{c}ij} \neq 0\}$, and
$|\mathcal{X}_j| \neq 0$. We partition $\mathcal{X}_j$ into subsets
$\mathcal{X}_j^+ = \{i \in \mathcal{X}_j | w'_{\text{c}ij} = 6 \}$ and
$\mathcal{X}_j^- = \{i \in \mathcal{X}_j | w'_{\text{c}ij} = -6 \}$, and
translate $c_j$ to an ASP rule of the form $conj_j \leftarrow \bigwedge_{i \in
        \mathcal{X}_j^+} atom_i, \bigwedge_{i \in \mathcal{X}_j^-} (\text{not}\
    atom_i)$, where $atom_i$ is an atom for input $x_i$. The disjunctive nodes are
interpreted differently depending on the desired policy type.

\begin{enumerate}[(4.a)]
    \item\label{sec:problog-extraction} \textbf{Stochastic policy - ProbLog rules}: %
          We use ProbLog's annotated disjunctions to represent mutually
          exclusive action probabilities. Each unique achievable activation of
          the conjunctive layer $\mathbf{c}^{(m)}\in \{-1, 1\}^{C'}$ with $1 \le
              m \le 2^{C'}$ \footnote{$C'$ is the number of remaining conjunctive
              nodes after pruning, which may differ from the initial choice of $C$.}
          forms the body of a unique annotated disjunction of the form $p_1 ::
              action_1 ; ... ;$ $p_N :: action_N \leftarrow \bigwedge_{i \in
                  \mathcal{C}^{(m)+}} conj_i, \bigwedge_{i \in \mathcal{C}^{(m)-}}
              (\backslash + conj_i)$, where $\mathcal{C}^{(m)+}= \{ i | c^{(m)}_i =
              1\}$, $\mathcal{C}^{(m)-} = \{ i | c^{(m)}_i = -1\}$, and $p_j =
              (\tilde{y}^{(m)}_{j} + 1)/2$ (the probability assigned to the
          $j$\textsuperscript{th} action in the disjunctive activation for the
          $m^{\text{th}}$ unique activation). We compute such annotated
          disjunctions for all unique conjunctive activations. Listing
          \ref{code:sc-ws-ppo-ndnf-mtl4-1e5-aux-3673-problog} shows an example
          of ProbLog rules.

    \item \textbf{Deterministic policy - ASP rules}: Since the disjunctive layer
          is also thresholded, we translate each disjunctive node into a normal
          clause. For a disjunctive node $d_j$, we consider the set
          $\mathcal{C}_j = \{i \in \{1..C'\} | w'_{\text{d}ij} \neq 0\}$, and
          $|\mathcal{C}_j| \neq 0$. We partition $\mathcal{C}_j$ into subsets
          $\mathcal{C}_j^+ = \{i \in \mathcal{C}_j | w'_{\text{d}ij} = 6 \}$ and
          $\mathcal{C}_j^- = \{i \in \mathcal{C}_j | w'_{\text{d}ij} = -6 \}$,
          and translate $d_j$ to a formula of the form $disj_j \leftarrow
              (\bigvee_{i \in \mathcal{C}_j^+} conj_i) \vee (\bigvee_{i \in
                      \mathcal{C}_j^-} (\text{not}\ conj_i))$. In practice, the formula is
          represented as multiple rules with the same head in ASP. Listing
          \ref{code:sc-sn-ppo-ndnf-mt-l4-1e5-aux-5745-asp} shows an example of
          ASP rules.
\end{enumerate}

If there is an encoder before the neural DNF-MT actor in the overall
architecture, we perform a mandatory step of invented predicate discretisation
(step 0 in Figure~\ref{fig:post-training-process}) at the beginning of the
post-training process. We take the sign of the invented predicate $\tanh$
activations, converting them to $\pm 1$ to interpret them as bivalent logical
truth values of $\top$ or $\bot$. Each invented predicate is defined as a
minimal set of raw observations using an ASP optimisation function (step 5 in
Figure~\ref{fig:post-training-process}).

\noindent\label{sec:nbl-translation}\textbf{Neural-bivalent-logic translation.}\ \ %
The translation for deterministic policies is bidirectional and maintains truth
value equivalence: given an input tensor and its translated logical assignment,
the interpreted bivalent truth value of the neural DNF-MT model with only $\pm
6$-and-0-valued weights is the same as the logical valuation of its translated
ASP program, and vice versa.\footnote{This translation does not support
predicate invention.} A formal proof of this bidirectional claim is provided
in Appendix \ref{appendix:translation-proof}.

\section{Experiments}

\definecolor{codegreen}{rgb}{0,0.6,0}
\definecolor{codegray}{rgb}{0.5,0.5,0.5}
\definecolor{codepurple}{rgb}{0.58,0,0.82}
\definecolor{backcolour}{rgb}{0.95,0.95,0.92}

\lstdefinestyle{ip}{
    numberstyle=\tiny\color{codegray},
    basicstyle=\ttfamily\footnotesize,
    breakatwhitespace=false,
    breaklines=true,
    captionpos=b,
    keepspaces=true,
    numbers=left,
    numbersep=5pt,
    showspaces=false,
    showstringspaces=false,
    showtabs=false,
    tabsize=2,
    frame=single,
}
\lstset{style=ip}

\begin{figure*}[ht]
    \centering
    \includegraphics[width=\textwidth]{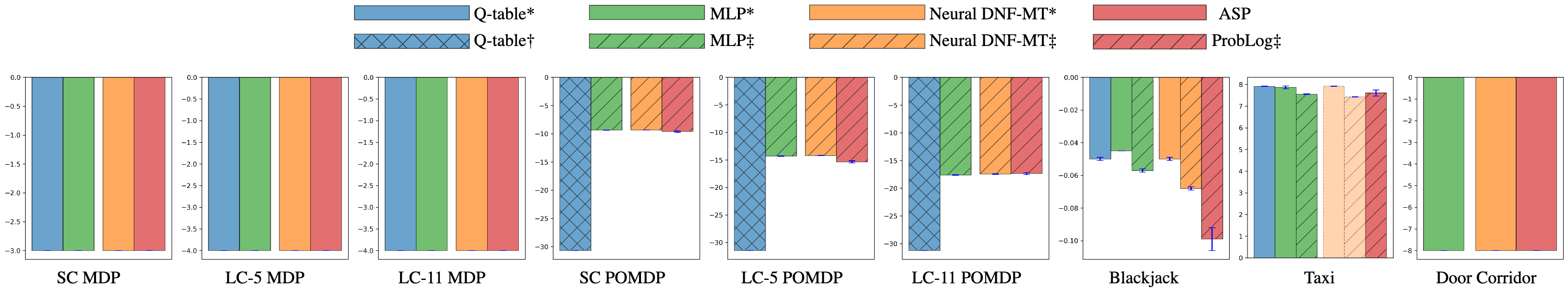}
    \caption{Mean episodic return (y-axis) $\pm$ standard error of the baselines
        and neural DNF-MT models, together with the ProbLog/ASP programs
        extracted from their corresponding neural DNF-MT models. All Q-tables
        are trained using Q-learning, and all MLP actors are trained with
        actor-critic PPO. Most neural DNF-MT actors are trained with
        actor-critic PPO. Except in the Taxi environment, the neural DNF-MT
        actor is distilled from a trained MLP actor (shown in dashed border and
        faded colour). Different symbols after the actor's name indicate
        different action selection methods: * for argmax action selection,
        $\dagger$ for $\epsilon$-greedy sampling, and $\ddagger$ for actor's
        distribution sampling. The same result is also reported in Table
        \ref{tab:rl-performance} in Appendix.}
    \label{fig:rl-results}
    \Description{Performance of the models in all environments, together with
        the ProbLog/ASP programs extracted from their corresponding neural
        DNF-MT models.}
\end{figure*}

We evaluate the RL performance (measured in episodic return) of our neural
DNF-MT actors and their interpreted logical policies in four sets of
environments with various forms of observations. Some tasks require stochastic
behaviours, while others can be solved with deterministic policies. We compare
our method with two baselines: Q-tables trained with Q-learning where applicable
and MLP actors trained with actor-critic PPO. Our neural DNF-MT actors are
trained with MLP critics using the PPO algorithm in the Switcheroo Corridor set,
Blackjack and Door Corridor environments. In the Taxi environment, we distil a
neural DNF-MT actor from a trained MLP actor. We do not directly evaluate the
extracted ProbLog policies because of the long ProbLog query time. Instead, we
evaluate their final neural DNF-MT actors before logical rule extraction (i.e.
after step 3, re-pruning) as an approximation. The approximation is acceptable
because a ProbLog policy's action distribution is the same as its corresponding
neural DNF-MT's action distribution to 3 decimal places. A performance
evaluation summary is shown in Figure \ref{fig:rl-results}, and a detailed
version is presented in Table \ref{tab:rl-performance} in Appendix.

\subsection{Switcheroo Corridor}\label{sec:ss-corridor}

We adopt an example environment from \cite{rl-book-sutton-barto} and create a
set of Switcheroo Corridor environments that support MDP tasks with
deterministic policies and POMDP tasks with stochastic policies. The observation
can be either (i) the state number one-hot encoding of the agent's current
position (an MDP task) or the wall status of the agent's current position (a
POMDP task). In most states, going left or right results in moving in the
intended direction. However, there are special states that reverse the action
effect. Thus, the nature of the task decides whether the optimal policy is
deterministic or stochastic. In the MDP setting, the optimal policy is
deterministic: identifying the special states based on the state number and
going left in them. In the POMDP setting, identifying the special states based
solely on wall status observations is impossible without a memory. The optimal
policy shows stochastic behaviour so that the correct action may be sampled in
the special states.

\begin{wrapfigure}{r}{0.5\columnwidth}
    \centering
    \includegraphics[width=0.47\columnwidth]{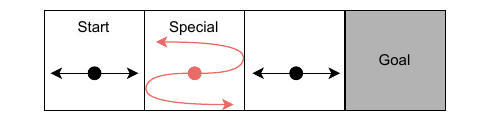}
    \caption{Small corridor (SC), same as the one from
        \cite{rl-book-sutton-barto}.}
    \label{fig:small-corridor}
    \Description{Small Corridor environment - an image visualising the
        environment. The agent starts at state 0 and needs to reach state 3
        (0-indexed), and state 2 is the special state that reverses the action
        effect.}
\end{wrapfigure}


The start, goal, and special states are customisable but fixed once created
throughout training and inference. We create three corridors based on different
configurations: Small Corridor (SC) as shown in Figure \ref{fig:small-corridor},
Long Corridor-5 (LC-5), and Long Corridor-11 (LC-11), to test the actor's
learning ability when the environment complexity increases. We show the figures
of LC-5 and LC-11 in Appendix \ref{appendix:ss-corridor}, and their
configurations are listed below.

\begin{table}[h]
    \caption{Environment configurations for LC-5 and LC-11.}
    \label{tab:ss-corridor-env-set-config}
    \begin{tabular}{ccccc} \toprule
        Name  & \begin{tabular}[c]{@{}c@{}}Corridor\\ Length\end{tabular} & \begin{tabular}[c]{@{}c@{}}Start\\ State\end{tabular} & \begin{tabular}[c]{@{}c@{}}Goal\\ State\end{tabular} & \begin{tabular}[c]{@{}c@{}}Special\\ State(s)\end{tabular} \\ \midrule
        LC-5  & 5                                                         & 0                                                     & 4                                                    & {[}1{]}                                                    \\
        LC-11 & 11                                                        & 7                                                     & 3                                                    & {[}5, 6, 7, 8{]}                                           \\ \bottomrule
    \end{tabular}
\end{table}

The first six groups in Figure \ref{fig:rl-results} show the performance of all
models in the environment set. In MDP settings, all methods using argmax action
selection perform equally well, reaching the goal with the minimum number of
steps. In POMDP settings, MLP and neural DNF-MT actors perform better than
Q-table with $\epsilon$-greedy sampling as expected, with minor performance
differences. Neural DNF-MT actors provide interpretability via logical programs
compared to MLP actors. Listing \ref{code:sc-sn-ppo-ndnf-mt-l4-1e5-aux-5745-asp}
shows the ASP program for a neural DNF-MT actor in SC MDP, where state 1 is
identified as special. Listing
\ref{code:sc-ws-ppo-ndnf-mtl4-1e5-aux-3673-problog} shows the ProbLog rules for
a neural DNF-MT actor in SC POMDP. As shown in line 1 in Listing
\ref{code:sc-ws-ppo-ndnf-mtl4-1e5-aux-3673-problog}, the actor favours the
action going right when only the left wall is present, which only happens in
state 0. Line 2 shows the case when the agent is in either state 1 or 2 with no
wall on either side, and the actor shows close to 50-50 preference for both
actions. The logical representations for policies learned in LC-5 and LC-11 are
included in Appendix \ref{appendix:add-exp-switcheroo}.

\begin{lstlisting}[
    label={code:sc-sn-ppo-ndnf-mt-l4-1e5-aux-5745-asp},
    caption={ASP rules of a neural DNF-MT actor in SC MDP.}
]
action(left) :- in_s_1.    action(right) :- not in_s_1.
\end{lstlisting}

\begin{lstlisting}[
    label={code:sc-ws-ppo-ndnf-mtl4-1e5-aux-3673-problog},
    caption={ProbLog rules of a neural DNF-MT actor in SC POMDP.}
]
0.041::action(left) ; 0.959::action(right) :- left_wall_present, \+ right_wall_present.
0.581::action(left) ; 0.419::action(right) :- \+ left_wall_present, \+ right_wall_present.
\end{lstlisting}

\subsection{Blackjack} \label{sec:blackjack-env}

The Blackjack environment from \cite{rl-book-sutton-barto} is a simplified
version of the card game Blackjack, where the goal is to beat the dealer by
having a hand closer to 21 without going over. The agent sees the sum of its
hand, the dealer's face-up card, and whether it has a usable ace. It can choose
to hit or stick. The performance across the models is shown in the
7\textsuperscript{th} group in Figure \ref{fig:rl-results} and Table
\ref{tab:blackjack-rl-win-rate}. The baseline Q-table from
\cite{rl-book-sutton-barto} only shows a single action, so we only evaluate it
with argmax action selection. We evaluate the MLP and neural DNF-MT actors with
both argmax action selection and actor's distribution sampling. MLP actors with
argmax action selection perform better than their distribution sampling
counterparts, with a higher episodic return and win rate. The same is observed
for neural DNF-MT actors. The extracted ProbLog rules\footnotemark
\footnotetext{An extracted ProbLog program example is listed in Appendix
    \ref{appendix:add-exp-blackjack}.}
perform worse than their original neural DNF-MT actors (no post-training
processing), with a higher policy divergence from the Q-table from
\cite{rl-book-sutton-barto}. We observe a policy change from a trained neural
DNF-MT actor to its extracted ProbLog policy (shown in Figures
\ref{fig:blackjack-ndnf-mt-3191-soft-policy-cmp-q} and
\ref{fig:blackjack-ndnf-mt-3191-problog-policy-cmp-q} in Appendix) at the
thresholding stage during the post-training processing. This unwanted policy
change caused by thresholding leads to performance loss and persists in later
environments; we will discuss this issue further in Section
\ref{sec:discussions}.

{\small
\begin{table}[h]
    \caption{Blackjack: performance of MLP actors, neural DNF-MT actors, and the
        extracted ProbLog programs. Policy divergence measures the proportion of
        states where the argmax policy disagrees with the Q-table from
        \cite{rl-book-sutton-barto}.}
    \label{tab:blackjack-rl-win-rate}
    \begin{tabular}{llll} \toprule
        Model                                & Episodic return    & Win rate             & Policy Divergence                     \\ \midrule
        Q-table \cite{rl-book-sutton-barto}* & -0.050 $\pm$ 0.001 & 42.94\% $\pm$ 0.00\% & NA                                    \\
        MLP*                                 & -0.045 $\pm$ 0.000 & 43.24\% $\pm$ 0.02\% & \multirow{2}{*}{15.87\% $\pm$ 0.30\%} \\
        MLP$\ddagger$                        & -0.057 $\pm$ 0.001 & 42.84\% $\pm$ 0.02\% &                                       \\
        NDNF-MT*                             & -0.050 $\pm$ 0.001 & 42.82\% $\pm$ 0.06\% & \multirow{2}{*}{20.66\% $\pm$ 0.56\%} \\
        NDNF-MT$\ddagger$                    & -0.068 $\pm$ 0.001 & 42.17\% $\pm$ 0.03\% &                                       \\
        ProbLog$\ddagger$                    & -0.099 $\pm$ 0.007 & 40.79\% $\pm$ 0.31\% & 27.92\% $\pm$ 1.25\%                  \\ \bottomrule
    \end{tabular}
\end{table}
}

\subsection{Taxi} \label{sec:taxi-env}

In the Taxi environment (Figure \ref{fig:taxi-env} in Appendix) from
\cite{taxi-og-paper}, the agent controls a taxi to pick up a passenger first and
drop them off at the destination hotel. A state number is used as the
observation, and it encodes the taxi, passenger and hotel locations using the
formula $((taxi\_row * 5 + taxi\_col) * 5 + passenger\_loc) * 4 + destination$.
Apart from moving in four directions, the agent can pick up/drop off the
passenger, but illegally picking up/dropping off will be penalised. The
environment is designed for hierarchical reinforcement learning but is solvable
with PPO and without task decomposition. However, for both MLP actors and neural
DNF-MT actors, we find that the performance is more sensitive to PPO's
hyperparameters and fine-tuning the hyperparameters is more difficult than in
other environments. With the wrong set of hyperparameters, the actor settles at
a local optimal with a reward of -200: never perform `pickup'/`drop-off' and
move until the step limit (200 steps). The environment is complex due to its
hierarchical nature, and learning the task dependencies based on purely state
numbers proves to be difficult, as a 1-value change in the x/y coordinate of the
taxi results is a change of state number in 100s/10s. We successfully trained
MLP actors with actor-critic PPO but failed to find a working set of
hyperparameters to train neural DNF-MT actors. Instead, we distil a neural
DNF-MT actor from a trained MLP actor, taking the same observation as input and
aiming to output the exact action probabilities as the MLP oracle.

The performance is shown in the 8\textsuperscript{th} group in Figure
\ref{fig:rl-results}. Actors using argmax action selection perform better than
their distribution sampling counterparts. Again, we observe a performance drop
in extracted ProbLog rules.\footnote{An extracted ProbLog program example is
    listed in Appendix \ref{appendix:add-exp-taxi}. } With 300 unique possible
starting states, the extracted ProbLog rules are not guaranteed to finish in 200
steps: 2 out of the 10 ProbLog evaluations with action probabilities sampling
have truncated episodes. Across ten post-training-processed neural DNF-MT actors
with argmax action selection, there are an average of 3.3 unique starting states
where the models cannot finish the environment within 200 steps. De-coupling the
observation seems complicated and makes it hard to learn concise conjunctions,
thus increasing the error rate in the post-training processing.

\subsection{Door Corridor}\label{sec:door-corridor}

Inspired by Minigrid \cite{minigrid}, we design a corridor grid with a fixed
configuration called Door Corridor, as shown in Figure \ref{fig:door-corridor}.
The agent observes a $3\times3$ grid in front of it (as shown as the input in
Figure \ref{fig:image-ndnf-mt-ac}) and has a choice of four actions: turn left,
turn right, move forward, and toggle. The toggle action only changes the status
of a door right in front of the agent.


For this environment, we use the architecture shown in Figure
\ref{fig:image-ndnf-mt-ac}, where an encoder is shared between the actor and the
critic. The performance of MLP actors, neural DNF-MT actors and their extracted
ASP programs are shown in the last group in Figure \ref{fig:rl-results}. Both of
the neural actors learn the optimal deterministic policy.

\begin{wrapfigure}{r}{0.4\columnwidth}
    \centering
    \includegraphics[width=0.37\columnwidth]{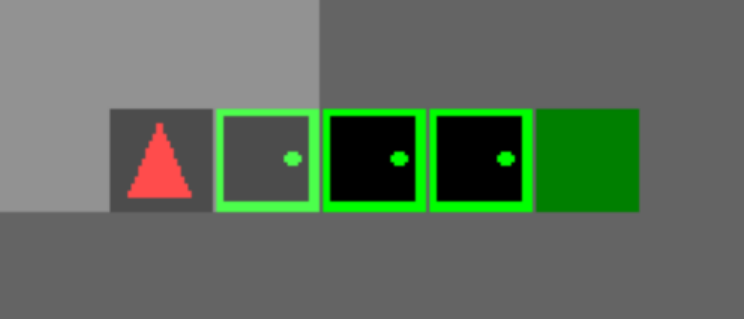}
    \caption{Door Corridor (DC): the agent needs to turn right first, and toggle
        and go through three doors to reach the end of the corridor.}
    \label{fig:door-corridor}
    \Description{Door Corridor environment - an image visualising the
        environment.}
\end{wrapfigure}

To evaluate an extracted ASP program in the environment, we first pass the
$3\times3$ observation to the encoder, convert invented predicates with bivalent
interpretations $\top$ to ASP facts, and then append these facts as context to
the base policy. The combined ASP program has to (i) output one stable model
with only one action at each step (logically mutual exclusive) and (ii) finish
without truncation to be counted as a successful run. These requirements are
also reflected in the neural DNF-MT actor: the final $\tanh$ activation should
only have one value greater than 0, and taking that only action with
greater-than-0 $\tanh$ activation at each step finishes the environment without
truncation. The auxiliary loss terms in Equations \ref{eq:l-feature-enc-reg},
\ref{eq:l-conj-tanh-out-reg} and \ref{eq:l-mutex-tanh} help the neural DNF-MT
actor to achieve these requirements but make the training less likely to
converge on a good solution. Out of 32 runs, 6 runs cannot finish the
environment within the step limit. However, 25 of the 26 remaining runs can be
interpreted as ASP policies. For the single failing case, it fails to maintain
logical mutual exclusivity after thresholding. While it is possible to extract a
ProbLog program from it, we know the environment supports an optimal
deterministic policy. Hence, no logical program is extracted for this run. The
ASP programs of the 25 runs successfully finish the environment with minimal
steps, as reported in Figure \ref{fig:rl-results}. Listing
\ref{code:dc-ndnf-mt-5187-asp} shows an example of the extracted ASP program
from one of the successful runs and a possible set of definitions for
the invented predicates.
\begin{lstlisting}[
    caption={An ASP policy for a neural DNF-MT actor in DC, with a possible set
        of definitions for the invented predicates.},
    label={code:dc-ndnf-mt-5187-asp}
]
action(turn_right) :- a_5, a_8.
action(forward) :- a_2.
action(toggle) :- a_3.
% Definitions of each invented predicate a_i:
a_2 :- top_right_corner_wall.
a_3 :- one_step_ahead_closed_door.
a_5 :- not curr_location_open_door,
        not one_step_ahead_closed_door.
a_8 :- two_step_ahead_unseen.
\end{lstlisting}

\noindent \textbf{Policy Intervention.} \label{sec:door-corridor-intervention} \ \ %
We create two variations of the base Door Corridor environment with different
termination conditions: Door Corridor-T (DC-T), where the agent must be in front
of the goal and toggle it instead of moving into it, and Door Corridor-OT
(DC-OT), where the agent must stand on the goal and take the action `toggle'.
The input observation remains unchanged since only the goal cell's mechanism
changes. The encoder can be reused immediately, but the actor and critic cannot.
An MLP actor trained on DC fails to finish within step limits in DC-T and DC-OT
environments without re-training. However, we can modify the ASP policy to
achieve the optimal reward in both DC-T and DC-OT environments. Listings
\ref{code:dct-ndnf-mt-5187-asp} and \ref{code:dcot-ndnf-mt-5187-asp} show the
modified ASP programs for DC-T and DC-OT environments, respectively. The
modified ASP programs can be ported back to neural DNF-MT actors by virtue of
the bidirectional neural-bivalent-logic translation. The modified neural DNF-MT
actors also finish DC-T and DC-OT environments with minimal steps without any
re-training.

\begin{lstlisting}[
    mathescape,
    caption={Policy for DC-T, modified from Listing
        \ref{code:dc-ndnf-mt-5187-asp}'s policy.},
    label={code:dct-ndnf-mt-5187-asp}
]
action(turn_right) :- a_5, a_8.
action(forward) :- $\textbf{not a\_1}$, a_2.
action(toggle) :- a_3.
$\textbf{action(toggle) :- a\_1, not a\_3, a\_12}.$
\end{lstlisting}

\begin{lstlisting}[
    mathescape,
    caption={Policy for DC-OT, modified from Listing
        \ref{code:dc-ndnf-mt-5187-asp}'s policy.},
    label={code:dcot-ndnf-mt-5187-asp}
]
action(turn_right) :- a_5, a_8, $\textbf{a\_11}$.
action(forward) :- a_2.
action(toggle) :- a_3.
$\textbf{action(toggle) :- not a\_2, not a\_3, not a\_11}.$
\end{lstlisting}

\section{Discussions} \label{sec:discussions}

We first analyse the persistent performance loss issue in Blackjack, Taxi, and
Door Corridor environments.

\noindent \textbf{Performance loss due to thresholding.}\ \ %
The thresholding step converts the target layer(s) from a weighted continuous
space to a discrete space with only 0 and $\pm 6$ values, saturating the $\tanh$
activation at $\pm 1$ and enabling the translation to bivalent logic. However,
the thresholding step may not maintain the same logical interpretation of the
layer output. Here we show this issue through an example in Listing
\ref{code:dc-2501-failed-case}, where a thresholded neural DNF-MT actor fails to
maintain logical mutual exclusivity in the Door Corridor environment. Note that
we apply thresholding on both the conjunctive and disjunctive layers since we
desire a deterministic policy.

\begin{lstlisting}[
    label={code:dc-2501-failed-case},
    caption={A neural DNF-MT actor in the Door Corridor environment that fails
        at the thresholding stage. We leave the bias terms uncalculated for
        brevity.}
]
% Conjunctive nodes:
c_0  =  3.03 x_7  + bias_c0
c_7  =  0.56 x_13 + bias_c7
c_9  = -1.56 x_2  + bias_c9
c_11 = -1.05 x_9  + bias_c11
% Disjunctive nodes:
d_1  =  4.58 c_0 + bias_d1
d_2  = -3.48 c_9 + bias_d2
d_3  =  1.29 c_7 + 0.76 c_9 + 4.33 c_11 + bias_d3
% Thresholded nodes with tau = 0 (and ASP translation):
c_0  =  6 x_7                  (conj_0  :- a_7.)
c_7  =  6 x_13                 (conj_7  :- a_13.)
c_9  = -6 x_2                  (conj_9  :- not a_2.)
c_11 = -6 x_9                  (conj_11 :- not a_9.)
d_1  =  6 c_0                  (act(1)  :- conj_0.)
d_2  = -6 c_9                  (act(2)  :- not conj_9.)
d_3  =  6 c_7 + 6 c_9 + 6 c_11 + 12
(act(3)  :- conj_7. act(3) :- conj_9. act(3) :- conj_11.)
\end{lstlisting}
The 1\textsuperscript{st} row of values in Table \ref{tab:dc-2501-failed-case}
are the pre-thresholding $\tanh$ output when $x_2 = -1, x_7 = 1, x_9 = 1, x_{13}
= -1$ , with only $\hat{y_1}$ interpreted as $\top$ and chosen as action. The
thresholded conjunctive nodes in the 2\textsuperscript{nd} row share the same
sign as row 1, but $\hat{y_3}$ becomes positive after thresholding, resulting in
two actions being $\top$ and thus violating logical mutual exclusivity. The
disjunctive nodes' original weights achieve the balance of importance between
$c_7$, $c_9$ and $c_{11}$ to make $\hat{y_3}$ negative. However, the
thresholding process ignores the weights and makes them equally important,
leading to a different output and truth value. It shows that the thresholding
stage cannot handle volatile and interdependent weights and maintain the
underlying truth table represented by the model. We leave it as a future work to
improve/replace the thresholding stage with a more robust method.

\begin{table}[H]
    \caption{Conjunctive and disjunctive nodes' $\tanh$ output when $x_2 = -1,
            x_7 = 1, x_9 = 1, x_{13} = -1$, calculated based on the formulation
        in Listing \ref{code:dc-2501-failed-case}. Row 1 is the original
        output without applying thresholding, and row 2 is the output after
        thresholding on value 0.}
    \label{tab:dc-2501-failed-case}
    \begin{tabular}{lllllll} \toprule
        $c_0$ & $c_7$ & $c_9$ & $c_{11}$ & $\hat{y_1}$ & $\hat{y_2}$ & $\hat{y_3}$ \\ \midrule
        1.00  & -0.51 & 0.92  & -0.78    & 1.00        & -1.00       & -0.86       \\
        1.00  & -1.00 & 1.00  & -1.00    & 1.00        & -1.00       & 1.00        \\ \bottomrule
    \end{tabular}
\end{table}

\balance

We now discuss the implications of our work in terms of performance,
interpretability, inference, and policy intervention.

\noindent \textbf{Performance.}\ \  From the experiments, we see that the neural
DNF-MT actor can be trained with actor-critic PPO or distilled from an MLP
oracle to learn an optimal policy in the Switcheroo Corridor, Taxi, and Door
Corridor environments. In Blackjack, the performance is worse but close to an
optimal MLP actor. Furthermore, we demonstrate that an encoder for handling
complex observations and realising predicate invention can be end-to-end trained
with the neural DNF-MT actor in the Door Corridor environment.

\noindent \textbf{Interpretability.}\ \  The logical programs extracted from
trained neural DNF-MT actors provide interpretability, which MLP actors lack. We
also demonstrate through different environments that we can represent stochastic
and deterministic policies in different forms of logic (ProbLog and ASP,
respectively).

\noindent \textbf{Inference.}\ \  Even if the actor has learnt an interpretable
policy, running a fully logic-based agent might not be efficient. Inference in
neural DNF-MT actors is significantly faster than in ProbLog or ASP, thanks to
tensor operations and environment parallelism (see Appendix
\ref{appendix:add-disc-run-time} for detailed comparison).

\noindent \textbf{Policy Intervention.}\ \  The bidirectional
neural-bivalent-logic translation allows us to modify the ASP program and
translate it back to the neural architecture without re-training, as shown in
Door Corridor's variations in Section \ref{sec:door-corridor-intervention}. This
feature could be helpful in tasks where we have background knowledge. By
pre-encoding the information into logical rules or modifying the logical rules
of an actor trained in a similar environment, the edited logic program can be
ported back to the neural model to provide a hot start in training. This
functionality will be explored in future work.

In summary, our neural DNF-MT model learns interpretable and editable policies,
with the neural benefits of end-to-end training and parallelism in inference and
the logical benefits of interpretable logical program representation.

\section{Related Work}

Many neuro-symbolic approaches perform the task of inductive logic programming
(ILP) \cite{ilp, ilp-30} in differentiable models, and policies are learned and
represented as logical rules. They are commonly applied in Relational RL
\cite{rrl, rrl-2} domains that utilise symbolic representations for states,
actions, and policies. NLRL \cite{nlrl} and NUDGE \cite{nudge} are two
approaches based on the differentiable ILP system \cite{delta-ilp} and its
extension from \cite{dilp-structured}, where the search space needs to be
defined first. NLRL generates candidate rules using rule templates. NUDGE
distils symbolic policy from a trained neural model by defining its search space
with mode declarations \cite{mode-declaration} and then training rule-associated
weights. NeSyRL \cite{nsrl-fol} and Differentiable Logic Machine (DLM)
\cite{dlm} do not associate weights with rules but predicates; thus, they are
not reliant on rule templates or mode declarations. NeSyRL uses a disjunctive
normal form Logical Neural Network (LNN) \cite{lnn} as its actor, and each
neuron represents an atom/logical connective. A pre-trained semantic parser
extracts first-order logic predicates from text-based observations, and the LNN
selects actions to generate trajectories that get stored in a replay buffer for
training, similar to DQN \cite{dqn}. DLM builds upon Neural Logic Machine
\cite{nlm} to realise forward chaining, but with logical computation units to
provide interpretability. A DLM actor is trained with actor-critic PPO
\cite{ppo}, with a specially designed critic with GRUs to handle different-arity
predicates.

Different from NLRL \cite{nlrl} and NUDGE \cite{nudge}, our neural DNF-MT model
does not use rule templates or mode declarations. Therefore, it does not rely on
human engineering to construct the inductive bias and can learn a wider range of
rules. Compared to the mentioned works that either operate on relational-based
observations \cite{nlrl, nudge, dlm} or require pre-trained networks to extract
logical predicates \cite{nudge, nsrl-fol}, we demonstrate that our neural DNF-MT
model is end-to-end trainable with preceding layers for predicate invention.
Akin to DLM \cite{dlm}, we use the PPO algorithm for training; however, our
method does not require a specialised critic.

\section{Conclusion}

We propose a neuro-symbolic approach named the neural DNF-MT model for learning
interpretable and editable policy in RL. It can be trained with actor-critic PPO
or distilled from a trained MLP actor, and an encoder for predicate invention
can also be end-to-end trained together with it. The trained neural DNF-MT model
can be represented as either a ProbLog program for stochastic policy or an ASP
program for deterministic policy. The neural-bivalent-logic translation is
bidirectional, allowing policy intervention by modifying the ASP program and
then converting it back to the neural model for efficient inference in parallel
environments. We evaluate the neural DNF-MT model in four environments with
different forms of observations and stochastic/deterministic behaviours. The
experiments show the neural DNF-MT model's capability to learn the optimal
policy with performance similar to an MLP actor's. Furthermore, it provides
logical representation and use cases for policy intervention, neither of which
can be provided easily by an MLP. In future work, we aim to follow up on the
policy intervention idea by providing the neural DNF-MT actor with a hot
starting point from a modified policy. Moreover, the thresholding stage during
post-training processing needs to be improved/replaced so that the underlying
logical relations learned by the neural DNF-MT model can be extracted without
performance loss.






\bibliographystyle{ACM-Reference-Format}
\bibliography{sample}

\newpage

\begin{appendix}
  \section{Neural-bivalent-logic Translation} \label{appendix:translation-proof}

This section focuses on proving that the neural-bivalent-logic translation
bidirectionally maintains the logical truth value equivalence. We first mention
some important properties of semi-symbolic nodes, then prove that the
translation from neural to bivalent logic has the same inference truth value and
vice versa. The proof is done for a conjunctive node, and a disjunctive node's
proof is similar.

\subsection{Semi-symbolic Node Properties}

Consider a node in a semi-symbolic layer with $I$ inputs, and its weights
$\mathbf{w} \in \{-6, 0, 6\}^I$ and at least one is non-zero (i.e. $\exists i
    \in \{1..I\}. w_i \neq 0$), it has either of the following two properties:

\begin{remark} \label{remark:conj-node-output}

    If the node is \textbf{conjunctive}, with $I$ inputs and at least one
    non-zero weight: given an input tensor $\mathbf{x}$ to it, and $\forall i
        \in \{1..I\}. w_i \in \{-6, 0, 6\}, x_i \in \{-1, 1\}$, the node's raw
    output will never be 0. It can be 6 or less than or equal to -6. Since
    $\tanh(\pm 6) = \pm 0.99998771165$, the conjunctive node's $\tanh$
    activation is treated to be exactly $\pm 1$ for the forward evaluation; or
    we can replace the $\tanh$ activation with a step function that maps the
    range to $\{-1, 1\}$,  as mentioned in the thresholding process in Section
    \ref{sec:thresholding}.
\end{remark}

\begin{remark}

    If the node is \textbf{disjunctive}, with $I$ inputs and at least one
    non-zero weight: given an input tensor $\mathbf{x}$ to it, and $\forall i
        \in \{1..I\}. w_i \in \{-6, 0, 6\}, x_i \in \{-1, 1\}$, the node's raw
    output will never be 0. It can be -6 or greater or equal to 6. Like the
    conjunctive node, the disjunctive node's activation is $\pm 1$.

\end{remark}

We focus on solving the combined program of the rule translated from a
semi-symbolic node with weights $\mathbf{w} \in \{-6, 0, 6\}^I$, and the facts
translated from an input $\mathbf{x} \in \{-1, 1\}^I$. If an input/conjunction
is connected to a conjunctive/disjunctive node with a weight of value $-6$, it
can be translated to a literal with either classical negation (`-') or negation
as failure (NAF, `not') and then added as part of the rule body. Similarly, any
$x_i = -1$ can be either mapped to a fact with classical negation `-a\_i', or
not added to the program and `not a\_i' would be true. The inference result of
the rules is the same regardless of whether the classical negation or NAF
translation is chosen. We choose the NAF style of translation across the paper
and proof.

\subsection{Neural to Bivalent Logic Translation with Truth Value Equivalence}

\renewcommand\qedsymbol{$\blacksquare$}

We have a conjunctive node with $I$ inputs in a semi-symbolic layer, and its
weights are $\mathbf{w}$ and $\delta = 1$. Let an input to the node be
$\mathbf{x}$ and $g(\mathbf{x})$ be the raw output of the conjunctive node
without any activation. The following conditions also hold for this node:
\begin{align}
    \mathbf{w} \in \{-6, 0, 6\}^I \label{condition:neural-to-asp-weight}, & \exists i \in \{1..I\}. w_i \neq 0                    \\
    \mathbf{x}                                                            & \in \{-1, 1\}^I \label{condition:neural-to-asp-input}
\end{align}

Recall the bivalent logic translation we defined in Definition \ref{def:gbl}.
Let $b$ be the bivalent logical interpretation of this conjunctive node. And $b
    = \top$ if $\tanh(g(\mathbf{x})) > 0$, and $b = \bot$ otherwise.

We specifically consider the set $\mathcal{X} = \{i | i \in \{1..I\}, w_i \neq
    0\}$, and $|\mathcal{X}| \neq 0$. By construction, $\forall i \in \mathcal{X}.
    |w_i| = 6$. We further divide set $\mathcal{X}$ into two mutually exclusive
subsets:
\begin{align*}
    \mathcal{X}^+ & = \left \{i | i \in \mathcal{X}, w_i = 6\right \}  \\
    \mathcal{X}^- & = \left \{i | i \in \mathcal{X}, w_i = -6\right \}
\end{align*}

By construction, $\mathcal{X}^+ \cap \mathcal{X}^- = \emptyset$ and
$|\mathcal{X}^+| + |\mathcal{X}^-| = |\mathcal{X}|$.

We translate the conjunctive node to an ASP rule, and the translated rule is in
the form:
\begin{equation} \label{eq:translation-neural-to-asp-rule}
    h \leftarrow \bigwedge_{i \in \mathcal{X}^+} a_i, \bigwedge_{i \in \mathcal{X}^-} (\text{not}\ a_i)
\end{equation}

We translate an input $\mathbf{x} \in \{-1, 1\}^I$ to facts, and the set of
facts are:
\begin{equation} \label{eq:translation-neural-to-asp-facts}
    \{a_i. | i \in \{1..I\}, x_i = 1\}
\end{equation}

\begin{proposition}
    Given a conjunctive semi-symbolic node that satisfies Conditions
    (\ref{condition:neural-to-asp-weight}) and its translated ASP rule with rule
    head $h$ based on Translation (\ref{eq:translation-neural-to-asp-rule}), and
    an input tensor $\mathbf{x}$ that satisfies Condition
    (\ref{condition:neural-to-asp-input}) and its translated facts based on
    Translation (\ref{eq:translation-neural-to-asp-facts}), the bivalent logical
    interpretation $b$ of the conjunctive semi-symbolic node from the input
    tensor $\mathbf{x}$ is the same as the truth value of the rule head $h$
    evaluated with the joint ASP program of the translated rule and facts.
\end{proposition}

\begin{proof}
    Under Condition(\ref{condition:neural-to-asp-weight}), the bias of the
    conjunctive node is calculated as follows, using Eq~(\ref{eq:ss-bias}):
    \begin{align*}
        \beta = \max_{i=1}^{I} |w_i| - \sum_{i=1}^{I} |w_i| = 6 - \sum_{i \in \mathcal{X}} 6 = 6 - 6|\mathcal{X}|
    \end{align*}

    So we calculate the output of the conjunctive node as follows:
    \begin{align*}
        g(\mathbf{x}) & = \sum_{i=1}^{I} w_i x_i + \beta                                                             \\
                      & = 6 \sum_{i \in \mathcal{X}^+} x_i - 6 \sum_{i \in \mathcal{X}^-} x_i + (6 - 6|\mathcal{X}|)
    \end{align*}

    Based on the input values specified in
    Condition~(\ref{condition:neural-to-asp-input}), we can further split $\mathcal{X}^+$
    and $\mathcal{X}^-$ into four mutually exclusive subsets:
    \begin{align*}
        \mathcal{X}^{++}           & = \{i | i \in \mathcal{X}^+ , x_i = 1\}  \\
        \mathcal{X}^{+-}           & = \{i | i \in \mathcal{X}^+ , x_i = -1\} \\
        \mathcal{X}^{-\phantom{}-} & = \{i | i \in \mathcal{X}^- , x_i = -1\} \\
        \mathcal{X}^{-+}           & = \{i | i \in \mathcal{X}^- , x_i = 1\}
    \end{align*}

    By construction, we have $\mathcal{X}^{++} \cap \mathcal{X}^{+-} = \emptyset$, $\mathcal{X}^{-\phantom{}-}
        \cap \mathcal{X}^{-+} = \emptyset$, and the followings:
    \begin{align*}
        |\mathcal{X}^{++}| + |\mathcal{X}^{+-}|                                                     & = |\mathcal{X}^+| \\
        |\mathcal{X}^{-\phantom{}-}| + |\mathcal{X}^{-+}|                                           & = |\mathcal{X}^-| \\
        |\mathcal{X}^{++}| + |\mathcal{X}^{+-}| + |\mathcal{X}^{-\phantom{}-}| + |\mathcal{X}^{-+}| & = |\mathcal{X}|
    \end{align*}

    Thus, we can calculate the output of the conjunctive node as:
    \begin{align}
        g(\mathbf{x}) & = 6 \sum_{i \in \mathcal{X}^+} x_i - 6 \sum_{i \in \mathcal{X}^-} x_i + (6 - 6|\mathcal{X}|) \nonumber                                                                                                 \\
                      & =6 \left( \sum_{i \in \mathcal{X}^{++}} 1 + \sum_{i \in \mathcal{X}^{+-}} (-1) -  \sum_{i \in \mathcal{X}^{-\phantom{}-}}(-1) - \sum_{i \in \mathcal{X}^{-+}} 1  + 1 - |\mathcal{X}| \right) \nonumber \\
                      & = 6(|\mathcal{X}^{++}| - |\mathcal{X}^{+-}| + |\mathcal{X}^{-\phantom{}-}| - |\mathcal{X}^{-+}| - |\mathcal{X}| + 1) \nonumber                                                                         \\
                      & = 6(1 - 2|\mathcal{X}^{+-}| - 2|\mathcal{X}^{-+}|) \label{proof:1-ssl-output}
    \end{align}

    Now, based on Remark~\ref{remark:conj-node-output}, we consider the two case
    of $g(\mathbf{x})$.

    \phantom{}

    \textbf{Case 1}: $g(\mathbf{x}) = 6 > 0$. The conjunctive node's bivalent
    logical interpretation $b = \top$. We need to prove that $h$ from
    Translation~(\ref{eq:translation-neural-to-asp-rule}) is in the answer set
    (i.e. $h$ is true).

    By $g(\mathbf{x}) = 6$ and Eq~(\ref{proof:1-ssl-output}):
    \begin{align}
        1 - 2|\mathcal{X}^{+-}| - 2|\mathcal{X}^{-+}|                          & = 1 \nonumber                         \\
        |\mathcal{X}^{+-}| + |\mathcal{X}^{-+}|                                & = 0 \nonumber                         \\
        \Rightarrow \mathcal{X}^{++}=\mathcal{X}^+, \mathcal{X}^{-\phantom{}-} & = \mathcal{X}^-\label{proof:1-case-1}
    \end{align}

    Combine Eq~(\ref{proof:1-case-1}) with
    Translation~(\ref{eq:translation-neural-to-asp-facts}), the ASP
    facts are:
    \begin{equation} \label{proof:1-case-1-asp-facts-program}
        \{ a_i. | i \in \mathcal{X}^+ |\}
    \end{equation}

    Since:
    \begin{align*}
        \forall i \in \mathcal{X}^+. a_i \text{ is true}             & \Rightarrow \bigwedge_{i \in \mathcal{X}^+} a_i \text{ is true}               \\
        \forall i \in \mathcal{X}^-. \text{not}\ a_i \text{ is true} & \Rightarrow \bigwedge_{i \in \mathcal{X}^-} (\text{not}\ a_i) \text{ is true}
    \end{align*}

    The combined program of (\ref{proof:1-case-1-asp-facts-program}) and
    (\ref{eq:translation-neural-to-asp-rule}) has the answer set with $h$ in it
    (i.e. $h$ is true).

    \phantom{}

    \textbf{Case 2}: $g(\mathbf{x}) \leq -6 < 0$. The conjunctive node's
    bivalent logical interpretation $b = \bot$. We need to prove that $h$ from
    Translation (\ref{eq:translation-neural-to-asp-rule}) is not in the answer
    set (i.e. $h$ is false).

    By $g(\mathbf{x}) \leq -6$ and Eq~(\ref{proof:1-ssl-output}):
    \begin{align}
        1 - 2|\mathcal{X}^{+-}| - 2|\mathcal{X}^{-+}| & \leq -1 \nonumber             \\
        |\mathcal{X}^{+-}| + |\mathcal{X}^{-+}|       & \geq 1 \label{proof:1-case-2}
    \end{align}

    We can split Eq~(\ref{proof:1-case-2}) into three cases:

    \phantom{}

    \textbf{(a)}: If $|\mathcal{X}^{+-}| = 0, |\mathcal{X}^{-+}| \geq 1$.

    By $|\mathcal{X}^{+-}| = 0$, we have $\mathcal{X}^{++} = \mathcal{X}^+$.
    Combined with Translation~(\ref{eq:translation-neural-to-asp-facts}), we
    have the ASP facts:
    \begin{equation} \label{proof:1-case-2a-asp-facts-program}
        \{ a_i. | i \in \mathcal{X}^+ \cup i \in \mathcal{X}^{-+}\}
    \end{equation}

    We have:
    \begin{align*}
        \forall i \in \mathcal{X}^+. a_i \text{ is true} & \Rightarrow \bigwedge_{i \in \mathcal{X}^+} a_i \text{ is true}                \\
        \exists i \in \mathcal{X}^-. a_i \text{ is true} & \Rightarrow \bigwedge_{i \in \mathcal{X}^-} (\text{not}\ a_i) \text{ is false}
    \end{align*}

    Thus, the combined program of (\ref{proof:1-case-2a-asp-facts-program}) and
    (\ref{eq:translation-neural-to-asp-rule}) has the answer set with $h$ not in
    it (i.e. $h$ is false).

    \phantom{}

    \textbf{(b)}: If $|\mathcal{X}^{+-}| \geq 1, |\mathcal{X}^{-+}| = 0$.

    By $|\mathcal{X}^{+-}| \geq 1$, we have $\mathcal{X}^{+-} \neq
        \mathcal{X}^+$. And by $|\mathcal{X}^{-+}| = 0$, we have
    $\mathcal{X}^{-\phantom{}-} = \mathcal{X}^-$. Combine with
    Translation~(\ref{eq:translation-neural-to-asp-facts}), we have the ASP
    facts:
    \begin{equation} \label{proof:1-case-2b-asp-facts-program}
        \{ a_i. | i \in \mathcal{X}^{++}\}
    \end{equation}

    We have:
    \begin{align*}
        \exists i \in \mathcal{X}^+. a_i \text{ is not true} & \Rightarrow \bigwedge_{i \in \mathcal{X}^+} a_i \text{ is false} \\
    \end{align*}

    Thus, the combined program of (\ref{proof:1-case-2b-asp-facts-program}) and
    (\ref{eq:translation-neural-to-asp-rule}) has the answer set with $h$ not in
    it (i.e. $h$ is false).

    \phantom{}

    \textbf{(c)}: If $|\mathcal{X}^{+-}| \geq 1, |\mathcal{X}^{-+}| \geq 1$.

    This case is a combination of the previous two cases, and we can similarly
    prove that $h$ is not in the answer set (i.e. $h$ is false).

    We have proved that the bivalent logical interpretation $b$ of the
    conjunctive semi-symbolic node equals the translated rule head $h$'s truth
    value in both cases.
\end{proof}

\subsection{Bivalent Logic to Neural Translation with Truth Value Equivalence}

There are $I$ different possible facts $a_1, a_2, ... a_I$. Consider an ASP rule
with head $h$ in the form of:
\begin{equation} \label{eq:translation-asp-to-neural-rule}
    h \leftarrow \bigwedge_{i \in \mathcal{X}^+} a_i, \bigwedge_{i \in \mathcal{X}^-} (\text{not}\ a_i)
\end{equation}
where $\mathcal{X}^+$ and $\mathcal{X}^-$ are the sets of IDs of the positive
and negative literals in the rule. The following conditions hold:
\begin{align}
    \mathcal{X}^+  \subseteq I, \mathcal{X}^- \subseteq I, & |\mathcal{X}^+| \geq 0, |\mathcal{X}^-| \geq 0 \label{condition:asp-to-neural-1} \\
    \mathcal{X}^+ \cap \mathcal{X}^-                       & = \emptyset \label{condition:asp-to-neural-2}
\end{align}

We define the set $\mathcal{X} = \mathcal{X}^+ \cup \mathcal{X}^-$, and
$|\mathcal{X}| \neq 0$.

This rule can be translated to a conjunctive semi-symbolic node with $I$ inputs
and $\delta = 1$, and the weights are:
\begin{equation} \label{eq:translation-asp-to-neural-weights}
    \forall i \in \{1..I\}.\ \ w_i = \begin{cases}
        6  & \text{if } i \in \mathcal{X}^+ \\
        -6 & \text{if } i \in \mathcal{X}^- \\
        0  & \text{otherwise}
    \end{cases}
\end{equation}

Again, let $b$ be the bivalent logical interpretation of the conjunctive node as
defined in Definition \ref{def:gbl}.

We translate an ASP program of a set of facts $\mathcal{P} = \{a_i. | i \in
    \{1..I\}\}$ to an input tensor $\mathbf{x}$ as:
\begin{equation} \label{eq:translation-asp-to-neural-input}
    \forall i \in \{1..I\}.\ \ x_i = \begin{cases}
        1  & \text{if } a_i \in \mathcal{P} \\
        -1 & \text{otherwise}
    \end{cases}
\end{equation}

\begin{proposition}
    Given an ASP rule in the form of (\ref{eq:translation-asp-to-neural-rule})
    with head $h$ that satisfies Conditions (\ref{condition:asp-to-neural-1})
    and (\ref{condition:asp-to-neural-2}) and its translated conjunctive
    semi-symbolic node with weights based on Translation
    (\ref{eq:translation-asp-to-neural-weights}), and the set of facts ASP
    program and its translated input tensor based on Translation
    (\ref{eq:translation-asp-to-neural-input}), the truth value of $h$ evaluated
    with the joint ASP program of the rule and facts is the same as the bivalent
    logical interpretation $b$ of the translated conjunctive semi-symbolic node
    computed from the translated input tensor.
\end{proposition}

\begin{proof}
    Let the translated conjunctive node's output be $g(\mathbf{x})$ and the
    bias of the node be $\beta$. Under the Conditions of
    (\ref{condition:asp-to-neural-1}) and (\ref{condition:asp-to-neural-2}), the
    bias is calculated as:
    \begin{align*}
        \beta = \max_{i=1}^{I} |w_i| - \sum_{i=1}^{I} |w_i| = 6 - 6|\mathcal{X}|
    \end{align*}

    So the output of the conjunctive node is:
    \begin{align} \label{eq:asp-to-neural-ss-output}
        g(\mathbf{x}) & = \sum_{i=1}^{I} w_i x_i + \beta \nonumber                                                   \\
                      & = 6 \sum_{i \in \mathcal{X}^+} x_i - 6 \sum_{i \in \mathcal{X}^-} x_i + (6 - 6|\mathcal{X}|)
    \end{align}

    We consider two cases:

    \phantom{}

    \textbf{Case 1}: $h$ is in the answer set (i.e. $h$ is true).

    From it, we know that:
    \begin{align*}
        \forall i \in \mathcal{X}^+. & a_i \text{ is true}             \\
        \forall i \in \mathcal{X}^-. & \text{not}\ a_i \text{ is true}
    \end{align*}

    Combined with Translation~(\ref{eq:translation-asp-to-neural-input}), we
    have:
    \begin{align*}
        \forall i \in \mathcal{X}^+. & x_i = 1  \\
        \forall i \in \mathcal{X}^-. & x_i = -1
    \end{align*}

    Substitute the values into the conjunctive node's output in
    Eq~(\ref{eq:asp-to-neural-ss-output}):
    \begin{align*}
        g(\mathbf{x}) & = 6 \sum_{i \in \mathcal{X}^+} 1 - 6 \sum_{i \in \mathcal{X}^-} (-1) + (6 - 6|\mathcal{X}|) \\
                      & = 6(|\mathcal{X}^+| + |\mathcal{X}^-| + 1 - |\mathcal{X}|)
    \end{align*}

    By Condition~(\ref{condition:asp-to-neural-2}), we have $|\mathcal{X}^+| + |\mathcal{X}^-| =
        |\mathcal{X}|$. Thus:
    \begin{align*}
        g(\mathbf{x}) & = 6 > 0    \\
        b             & = \top = h
    \end{align*}

    \phantom{}

    \textbf{Case 2}: $h$ is not in the answer set (i.e. $h$ is false).

    We further split $\mathcal{X}^+$ and $\mathcal{X}^-$ into four mutually exclusive subsets:
    \begin{equation} \label{eq:asp-to-neural-case-2-subsets}
        \begin{split}
            \mathcal{X}^{++}           & = \{i | i \in \mathcal{X}^+ , a_i \in \mathcal{P}\}    \\
            \mathcal{X}^{+-}           & = \{i | i \in \mathcal{X}^+ , a_i \notin \mathcal{P}\} \\
            \mathcal{X}^{-\phantom{}-} & = \{i | i \in \mathcal{X}^- , a_i \notin \mathcal{P}\} \\
            \mathcal{X}^{-+}           & = \{i | i \in \mathcal{X}^- , a_i \in \mathcal{P}\}
        \end{split}
    \end{equation}

    For $h$ to be not in the answer set, the following must hold:
    \begin{align*}
        |\mathcal{X}^{+-}| + |\mathcal{X}^{-+}| \geq 1
    \end{align*}
    such that either or both of $\bigwedge_{i \in \mathcal{X}^{+-}} a_i$ and
    $\bigwedge_{i \in \mathcal{X}^{-+}} (\text{not}\ a_i)$ are false.

    Combine the subsets in (\ref{eq:asp-to-neural-case-2-subsets}) with
    Translation~(\ref{eq:translation-asp-to-neural-input}) and the input tensor
    $\mathbf{x}$ should be like:
    \begin{align*}
        \forall i \in \mathcal{X}^{++}.           & x_i = 1  \\
        \forall i \in \mathcal{X}^{+-}.           & x_i = -1 \\
        \forall i \in \mathcal{X}^{-\phantom{}-}. & x_i = -1 \\
        \forall i \in \mathcal{X}^{-+}.           & x_i = 1
    \end{align*}

    Substitute the values into the output of the conjunctive node in
    Eq~(\ref{eq:asp-to-neural-ss-output}), we have:
    \begin{align*}
        g(\mathbf{x}) & = 6 \left(\sum_{i \in \mathcal{X}^{++}} 1 + \sum_{i \in \mathcal{X}^{+-}} (-1)\right) - 6 \left(\sum_{i \in \mathcal{X}^{-\phantom{}-}} (-1) + \sum_{i \in \mathcal{X}^{-+}} 1\right)+ (6 - 6|\mathcal{X}|) \\
                      & = 6(|\mathcal{X}^{++}| - |\mathcal{X}^{+-}| + |\mathcal{X}^{-\phantom{}-}| - |\mathcal{X}^{-+}| - |\mathcal{X}| + 1)                                                                                        \\
                      & = 6(1 - 2|\mathcal{X}^{+-}| - 2|\mathcal{X}^{-+}|)
    \end{align*}

    Since $|\mathcal{X}^{+-}| + |\mathcal{X}^{-+}| \geq 1$, we have $1 - 2|\mathcal{X}^{+-}| - 2|\mathcal{X}^{-+}| \leq
        -1$. Hence:
    \begin{align*}
        g(\mathbf{x}) & \leq -6    \\
        \Rightarrow b & = \bot = h
    \end{align*}

    \phantom{}

    We have proved that the rule head $h$'s truth value equals the bivalent
    logical interpretation $b$ of the translated conjunctive semi-symbolic node
    in both cases.
\end{proof}

  \section{Training Details}

We use PyTorch \cite{pytorch} for implementing the neural DNF-MT actors and MLP
baselines, and CleanRL \cite{cleanrl} for the base implementation of
actor-criticy PPO. We also adopt the training technique of $\delta$ scheduling
for neural DNF-based models used in \cite{pix2rule} and
\cite{ns-classifications}. At each training iteration $i$, the $\delta$ value is
calculated as:

\begin{align*}
      m                   & = {\left\lfloor \frac{i - d}{s} \right\rfloor + 1} \\
      \delta_{\text{new}} & = \delta_{\text{initial}} * r ^ m
\end{align*}
where $d$ is the delay, $s$ is the `decay' step size, $\delta_{\text{initial}}$
is the initial $\delta$ value and $r$ is the `decay' rate, which all are
hyperparameters.

The implementation of neural DNF-based models is at
\url{https://github.com/kittykg/neural-dnf} and our training and evaluation code
is at \url{https://github.com/kittykg/neural-dnf-mt-policy-learning}.

\subsection{Base PPO Loss} \label{appendix:ppo-loss}

The PPO loss function based on the original PPO paper \cite{ppo} and CleanRL
\cite{cleanrl} is represented as:
\begin{align*}
      L^{\text{PPO}}(\theta)         & = \mathbb{E}_t \left[
            L^{\text{CLIP}}(\theta) +
            c_1 L^{\text{value}}(\theta) - c_2 S[\pi_\theta](s_t)
      \right]                                                                  \\
      \intertext{where}
      r_t(\theta)                    & = \frac{\pi_\theta(a_t | s_t)}
      {\pi_{\theta_{\text{old}}}(a_t | s_t)},                                  \\
      L^{\text{CLIP}}(\theta)        & = \mathbb{E}_t \left[
      \min \left( r_t(\theta) \hat{A}_t,
      \text{clip}(r_t(\theta), 1 - \epsilon, 1 + \epsilon) \hat{A}_t \right)
      \right],                                                                 \\
      S[\pi_\theta](s_t)             & = -\mathbb{E}_t \left[
            \sum_a \pi_\theta(a | s_t) \log \pi_\theta(a | s_t)
      \right],                                                                 \\
      L^{\text{unclipped v}}(\theta) & = (V_\theta(s_t) - V_t)^2,              \\
      V^{\text{clipped}}(\theta)     & = V_{\text{old}}(s_t) +
      \text{clip}(V_\theta(s_t) - V_{\text{old}}(s_t), -\epsilon, \epsilon),   \\
      L^{\text{clipped v}}(\theta)   & = (V^{\text{clipped}}(\theta) - V_t)^2, \\
      L^{\text{value}}(\theta)       & = \mathbb{E}_t \left[
            \max \left( L^{\text{unclipped v}}(\theta),
            L^{\text{clipped v}}(\theta) \right) \right]
\end{align*}

\begin{itemize}
      \item $c_1$ and $c_2$: hyperparameters, coefficients of the value loss and
            the entropy
      \item $\pi_\theta$: policy of the model with parameter $\theta$
      \item $a_t$: action at time $t$
      \item $s_t$: state at time $t$
      \item $\hat{A}_t$: advantage estimate at time $t$
      \item $\text{clip}(x, a, b)$: function, clip $x$ to the range $[a, b]$
      \item $\epsilon$: hyperparameter, clipping range
      \item $S[\pi_\theta](s_t)$: entropy of the policy at state $s_t$
      \item $V_\theta(s_t)$: value predicted by the model with parameter
            $\theta$, during loss calculation
      \item $V_t$: value target at time $t$
      \item $V_{\text{old}}(s_t)$: value predicted by the model with parameter
            $\theta$, during collection
\end{itemize}

\subsection{Post-training Processing of Neural DNF-MT Model} \label{appendix:post-training-process}

The post-training processing is based on the procedure described in
\cite{pix2rule} and \cite{ns-classifications}. We modify the pruning and rule
extraction stages to be better fitted for policy learning. We provide
additional information for some stages in the post-training processing below.

\textbf{(1) Pruning}: In experiments, we pass over conj.-to-disj. edges
(weights) first before the input-to-conj. edges. We use $\tau_{\text{prune}} =
      1e-3$ for the removal check.

\label{alg:problog-rule-extraction}\textbf{(4.a) ProbLog rules extraction} -
pseudocode

\begin{enumerate}
      \item Condensation via logical equivalence
            \begin{itemize}
                  \item Find all the conjunctions that are logically equivalent,
                        i.e. check if the conjunctions are the same
            \end{itemize}
      \item Rule simplification based on experienced observations
            \begin{itemize}
                  \item Based on all possible inputs to the conjunctive layer,
                        compute all unique activations of the conjunctive layer.

                  \item If any conjunctive node's activation is always
                        interpreted as true/false, that conjunction does not
                        need to be in the annotated disjunction (unless there is
                        no other conjunction in the rule body).

                  \item This step is optional if we cannot enumerate all
                        possible input. If we can, it will reduce the number of
                        ProbLog annotated disjunctions.
            \end{itemize}
      \item Generate ProbLog annotated disjunctions based on experienced
            observations (as described in
            \hyperref[sec:problog-extraction]{post-training step (4.a)})
            \begin{itemize}
                  \item Compute the probabilities from the mutex-tanh output for
                        each unique conjunctive activation. The rule body is
                        translated from the conjunctive activation, and the
                        probabilities are used in the annotated disjunction
                        head.
            \end{itemize}
\end{enumerate}

  \section{RL Environment Details}\label{appendix:rl-env-details}

Our experiments use Gymnasium \cite{gymnasium} APIs and environments. Our custom
environments (Switcheroo Corridor and Door Corridor environments) are
implemented with Gymnasium APIs and are available at
\url{https://github.com/kittykg/corridor-grid}.

\subsection{Switcheroo Corridor}\label{appendix:ss-corridor}

A Switcheroo Corridor environment is created based on the sample environment
from \cite[Chapter 13.1]{rl-book-sutton-barto}. The reward is -1 for each step.
The episode is terminated when the agent reaches the goal state or is truncated
after the maximum number of steps. By default, the truncation step limit is 50.

Table \ref{tab:ss-corridor-env-set-config} shows the configuration of the three
Switcheroo Corridor environments used in Section \ref{sec:ss-corridor}. LC-5 and
LC-11 are shown in Figure \ref{fig:lc-5-env} and Figure \ref{fig:lc-11-env}
respectively.

\begin{figure}[H]
    \centering
    \includegraphics[width=0.8\columnwidth]{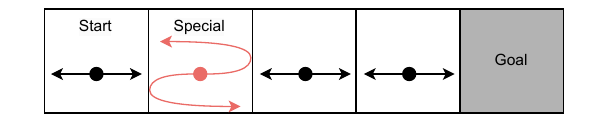}
    \caption{Long Corridor-5 environment, created according to the configuration
        shown in Table \ref{tab:ss-corridor-env-set-config}.}
    \label{fig:lc-5-env}
    \Description{Long Corridor-5 environment - an image visualising the
        environment.}
\end{figure}

\begin{figure}[H]
    \centering
    \includegraphics[width=\columnwidth]{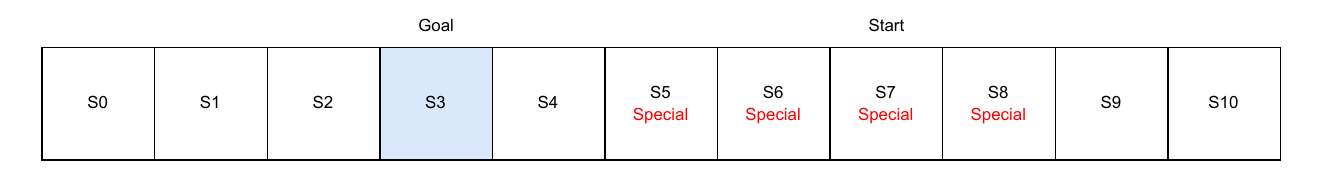}
    \caption{Long Corridor-11 environment, created according to the
        configuration shown in Table \ref{tab:ss-corridor-env-set-config}.}
    \label{fig:lc-11-env}
    \Description{Long Corridor-11 environment - an image visualising the
        environment.}
\end{figure}

\begin{table}[H]
    \caption{Environment configurations for the three environments in Switcheroo
        Corridor Environment Set. This table is almost identical to Table
        \ref{tab:ss-corridor-env-set-config}, but with the addition of SC's
        configuration.}
    \label{tab:ss-corridor-env-set-config-appendix}
    \begin{tabular}{ccccc} \toprule
        \begin{tabular}[c]{@{}c@{}}Name\\ (Short code)\end{tabular}        & \begin{tabular}[c]{@{}c@{}}Corridor\\ Length\end{tabular} & \begin{tabular}[c]{@{}c@{}}Start\\ State\end{tabular} & \begin{tabular}[c]{@{}c@{}}Goal\\ State\end{tabular} & \begin{tabular}[c]{@{}c@{}}Special\\ State(s)\end{tabular} \\ \midrule
        \begin{tabular}[c]{@{}c@{}}Small Corridor\\ (SC)\end{tabular}      & 4                                                         & 0                                                     & 3                                                    & {[}1{]}                                                    \\
        \begin{tabular}[c]{@{}c@{}}Long Corridor-5\\ (LC-5)\end{tabular}   & 5                                                         & 0                                                     & 4                                                    & {[}1{]}                                                    \\
        \begin{tabular}[c]{@{}c@{}}Long Corridor-11\\ (LC-11)\end{tabular} & 11                                                        & 7                                                     & 3                                                    & {[}5, 6, 7, 8{]}                                           \\ \bottomrule
    \end{tabular}
\end{table}

\subsection{Blackjack}

The Blackjack environment from \cite[Chapter 5.1]{rl-book-sutton-barto} is a
simplified version of the Blackjack card game. The standard Blackjack rules can
be found at \url{https://en.wikipedia.org/wiki/Blackjack}. The environment
terminates when the agent sticks or its hand exceeds 21, and termination never
happens. It receives a reward of +1 for winning, -1 for losing and 0 for
drawing.

Our experiments use the Gymnasium's implementation at
\url{https://Gymnasium.farama.org/environments/toy_text/blackjack/}. Instead of
using Gymnasium's tuple observation (the player's current sum, the dealer's
hand, and the player's useable ace), we convert each value to a one-hot encoding
but in the range $\{-1, 1\}$ and stack them together as the observation. There
are 44 bits in total in the final observation encoding. The observation input to
the neural DNF-MT actor is shown in Figure~\ref{fig:binary-ndnf-mt-ac}.

\subsection{Taxi}

The Taxi environment, purposed in \cite{taxi-og-paper}, has 4 coloured squares
where a hotel and the passenger will be allocated initially at the start. The
hotel location remains unchanged throughout the episode. The agent controls a
taxi to pick up the passenger and drop them off at the hotel. Each step gives a
-1 reward, but illegally executing `pickup' and `drop-off' gives a -10 reward.
If a passenger is successfully delivered to the hotel, the agent receives a +20
reward.

Our experiments use Gymnasium's implementation at
\url{https://Gymnasium.farama.org/environments/toy_text/taxi/}. By default, the
truncation step limit is 200.

\begin{figure}[H]
    \centering
    \includegraphics[width=0.8\columnwidth]{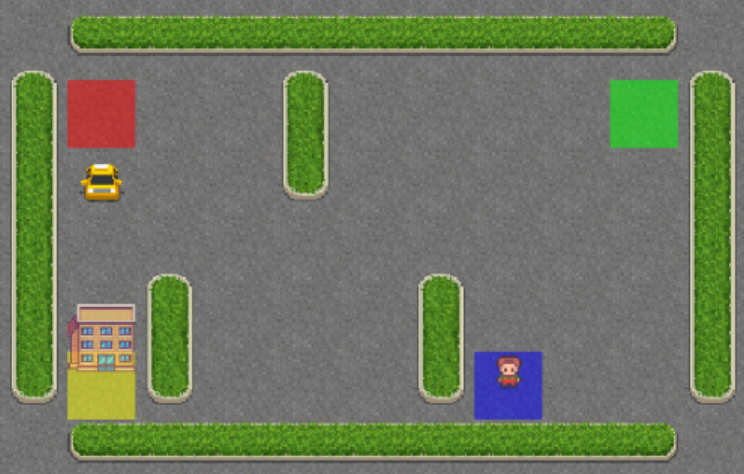}
    \caption{Taxi Environment: the taxi needs to pick up a passenger and drop
        them off at the destination hotel to finish the episode.}
    \label{fig:taxi-env}
    \Description{Taxi environment - an image visualising the environment.}
\end{figure}

\subsection{Door Corridor}

The Door Corridor environment (as shown in Figure~\ref{fig:door-corridor}) is a
custom environment created based on the MiniGrid environment \cite{minigrid}.
The implementation of the environment is based on Gymnasium's MiniGrid. We keep
the observation similar to MiniGrid's style: a 3$\times$3 grid with two channels
of objects and cell status. The agent receives a -1 reward for each step, and
the default truncation step limit is 270.

\begin{lstlisting}[
    caption={Door Corridor Environment action space and observation space.},
    label={lst:door-corridor-env},
    language=Python,
]
# Action space
class DoorCorridorAction(IntEnum):
    LEFT = 0
    RIGHT = 1
    FORWARD = 2
    TOGGLE = 3

# Observation space
class Object(IntEnum):
    UNSEEN = 0
    EMPTY = 1
    WALL = 2
    DOOR = 3
    AGENT = 4
    GOAL = 5

class State(IntEnum):
    OPEN = 0    # Most objects will be OPEN
    CLOSED = 1  # A door can be OPEN or CLOSED
\end{lstlisting}

The two variants, Door Corridor-T (DC-T) and Door Corridor-OT (DC-OT), only
differ in their termination check. DC-T terminates when the agent toggles the
goal cell without being in it, and DC-OT terminates when the agent stands on
the goal and then toggles.

Table \ref{tab:dc-optimal-policy} shows the optimal policy's actions for the
Door Corridor environment and its variants.

\begin{table}[H]
    \caption{Optimal policy's actions for the Door Corridor (DC) environment and
        its variants, with blue text showing the difference in the
        variants.}
    \label{tab:dc-optimal-policy}
    \begin{tabular}{ccc} \toprule
        Environment & \begin{tabular}[c]{@{}c@{}}Optimal policy's\\actions (text)\end{tabular}                                                                 & \begin{tabular}[c]{@{}c@{}}Optimal policy's\\actions (int)\end{tabular}                \\ \midrule
        DC          & \begin{tabular}[c]{@{}c@{}}[Right, Toggle,\\Forward, Toggle,\\Forward, Toggle,\\Forward, Forward]\end{tabular}                           & \begin{tabular}[c]{@{}c@{}}[1, 3, 2, 3,\\2, 3, 2, 2]\end{tabular}                      \\
        DC-T        & \begin{tabular}[c]{@{}c@{}}[Right, Toggle,\\Forward, Toggle,\\Forward, Toggle,\\Forward, \textcolor{blue}{Toggle}]\end{tabular}          & \begin{tabular}[c]{@{}c@{}}[1, 3, 2, 3,\\2, 3, 2, \textcolor{blue}{3}]\end{tabular}    \\
        DC-OT       & \begin{tabular}[c]{@{}c@{}}[Right, Toggle,\\Forward, Toggle,\\Forward, Toggle,\\Forward, Forward, \textcolor{blue}{Toggle}]\end{tabular} & \begin{tabular}[c]{@{}c@{}}[1, 3, 2, 3,\\2, 3, 2, 2, \textcolor{blue}{3}]\end{tabular} \\ \bottomrule
    \end{tabular}
\end{table}

  \section{Additional Experimental Results} \label{appendix:add-exp-results}

We use the clingo solver \cite{clingo} to run ASP programs.

Table \ref{tab:rl-performance} shows the performance of the models in all
environments.

\begin{table}[h]
    \caption{Performance of the models in all environments, together with the
        logic programs extracted from their corresponding neural DNF-MT actors.
        The episodic return metric is represented in the form of mean $\pm$
        standard error. For the Switcheroo Corridor environment set and Door
        Corridor environment, each model has 16 runs with different seeds. For
        Blackjack and Taxi environments, each model has 10 runs with different
        seeds. Each run is evaluated over 1,000,000 episodes and takes the
        average as the run's performance. The final metric is calculated over
        the list of each run's averaged episodic returns. All Q-tables are
        trained using Q-learning. Different symbols after the actor name
        indicate different action selection methods: * for argmax action
        selection, $\dagger$ for $\epsilon$-greedy sampling, and $\ddagger$ for
        actor's distribution sampling. }
    \label{tab:rl-performance}
    \begin{tabular}{lll} \toprule
        Environment                    & Actor Model                          & Episodic return     \\ \midrule
        \multirow{4}{*}{SC MDP}        & Q-table*                             & -3.000 $\pm$ 0.000  \\
                                       & MLP*                                 & -3.000 $\pm$ 0.000  \\
                                       & Neural DNF-MT*                       & -3.000 $\pm$ 0.000  \\
                                       & NDNF-MT: ASP*                        & -3.000 $\pm$ 0.000  \\ \hline
        \multirow{4}{*}{SC POMDP}      & Q-table$\dagger$                     & -30.683 $\pm$ 0.006 \\
                                       & MLP$\ddagger$                        & -9.342 $\pm$ 0.022  \\
                                       & Neural DNF-MT$\ddagger$              & -9.336 $\pm$ 0.019  \\
                                       & NDNF-MT: ProbLog$\ddagger$           & -9.607  $\pm$ 0.129 \\ \hline
        \multirow{4}{*}{LC-5 MDP}      & Q-table*                             & -4.000 $\pm$ 0.000  \\
                                       & MLP*                                 & -4.000 $\pm$ 0.000  \\
                                       & Neural DNF-MT*                       & -4.000 $\pm$ 0.000  \\
                                       & NDNF-MT: ASP*                        & -4.000 $\pm$ 0.000  \\ \hline
        \multirow{4}{*}{LC-5 POMDP}    & Q-table$\dagger$                     & -31.459 $\pm$ 0.004 \\
                                       & MLP$\ddagger$                        & -14.304 $\pm$ 0.047 \\
                                       & Neural DNF-MT$\ddagger$              & -14.212 $\pm$ 0.031 \\
                                       & NDNF-MT: ProbLog$\ddagger$           & -15.358 $\pm$ 0.193 \\ \hline
        \multirow{4}{*}{LC-11 MDP}     & Q-table*                             & -4.000 $\pm$ 0.000  \\
                                       & MLP*                                 & -4.000 $\pm$ 0.000  \\
                                       & Neural DNF-MT*                       & -4.000 $\pm$ 0.000  \\
                                       & NDNF-MT: ASP*                        & -4.000 $\pm$ 0.000  \\ \hline
        \multirow{4}{*}{LC-11 POMDP}   & Q-table$\dagger$                     & -31.235 $\pm$ 0.004 \\
                                       & MLP$\ddagger$                        & -17.625 $\pm$ 0.051 \\
                                       & Neural DNF-MT$\ddagger$              & -17.433 $\pm$ 0.056 \\
                                       & NDNF-MT: ProbLog$\ddagger$           & -17.361 $\pm$ 0.183 \\ \hline
        \multirow{6}{*}{Blackjack}     & Q-table \cite{rl-book-sutton-barto}* & -0.050 $\pm$ 0.001  \\
                                       & MLP*                                 & -0.045 $\pm$ 0.000  \\
                                       & MLP$\ddagger$                        & -0.057 $\pm$ 0.001  \\
                                       & Neural DNF-MT*                       & -0.050 $\pm$ 0.001  \\
                                       & Neural DNF-MT$\ddagger$              & -0.068 $\pm$ 0.001  \\
                                       & NDNF-MT: ProbLog$\ddagger$           & -0.099 $\pm$ 0.007  \\ \hline
        \multirow{6}{*}{Taxi}          & Q-table*                             & 7.913 $\pm$ 0.009   \\
                                       & MLP*                                 & 7.865 $\pm$ 0.066   \\
                                       & MLP$\ddagger$                        & 7.550 $\pm$ 0.011   \\
                                       & Neural DNF-MT*                       & 7.926 $\pm$ 0.009   \\
                                       & Neural DNF-MT$\ddagger$              & 7.424 $\pm$ 0.009   \\
                                       & NDNF-MT: ProbLog$\ddagger$           & 7.604 $\pm$ 0.139   \\ \hline
        \multirow{3}{*}{Door Corridor} & MLP*                                 & -8.000 $\pm$ 0.000  \\
                                       & Neural DNF-MT*                       & -8.000 $\pm$ 0.000  \\
                                       & NDNF-MT: ASP*                        & -8.000 $\pm$ 0.000  \\ \bottomrule
    \end{tabular}
\end{table}

\subsection{Switcheroo Corridor} \label{appendix:add-exp-switcheroo}

The model architectures are listed below:

\begin{table}[H]
    \begin{tabular}{ll} \toprule
        Model                                                         & Architecture                                                                                                                                                                \\ \midrule
        \begin{tabular}[c]{@{}l@{}}MLP actor\\MDP\end{tabular}        & \begin{tabular}[c]{@{}l@{}}(0): Linear(in=N\_STATES, out=4, bias=True)\\(1): Tanh()\\(2): Linear(in=4, out=2, bias=True)\end{tabular}                                       \\ \hline
        \begin{tabular}[c]{@{}l@{}}MLP actor\\POMDP\end{tabular}      & \begin{tabular}[c]{@{}l@{}}(0): Linear(in=2, out=4, bias=True)\\(1): Tanh()\\(2): Linear(in=4, out=2, bias=True)\end{tabular}                                               \\ \hline
        \begin{tabular}[c]{@{}l@{}}NDNF-MT\\actor\\MDP\end{tabular}   & \begin{tabular}[c]{@{}l@{}}(0): SemiSymbolic(\\\ \ \ \ \ \ in=N\_STATES, out=4, type=CONJ.)\\(1): SemiSymbolicMutexTanh(\\\ \ \ \ \ \ in=4, out=2, type=DISJ.)\end{tabular} \\ \hline
        \begin{tabular}[c]{@{}l@{}}NDNF-MT\\actor\\POMDP\end{tabular} & \begin{tabular}[c]{@{}l@{}}(0): SemiSymbolic(\\\ \ \ \ \ \ in=2, out=4, type=CONJ.)\\(1): SemiSymbolicMutexTanh(\\\ \ \ \ \ \ in=4, out=2, type=DISJ.)\end{tabular}         \\ \hline
        \begin{tabular}[c]{@{}l@{}}Critic\\MDP\end{tabular}           & \begin{tabular}[c]{@{}l@{}}(0): Linear(in=N\_STATES, out=64, bias=True)\\(1): Tanh()\\(2): Linear(in=64, out=1, bias=True)\end{tabular}                                     \\ \hline
        \begin{tabular}[c]{@{}l@{}}Critic\\POMDP\end{tabular}         & \begin{tabular}[c]{@{}l@{}}(0): Linear(in=2, out=64, bias=True)\\(1): Tanh()\\(2): Linear(in=64, out=1, bias=True)\end{tabular}                                             \\ \bottomrule
    \end{tabular}
\end{table}

The PPO hyperparameters used for training both the MLP actor and the neural
DNF-MT actor are listed below:

\begin{table}[H]
    \begin{tabular}{ll} \toprule
        Hyperparameter   & Value             \\ \midrule
        total\_timesteps & $1\mathrm{e}{5}$  \\
        learning\_rate   & $1\mathrm{e}{-2}$ \\
        num\_envs        & 8                 \\
        num\_steps       & 64                \\
        anneal\_lr       & True              \\
        gamma            & 0.99              \\
        gae\_lambda      & 0.95              \\
        num\_minibatches & 8                 \\
        update\_epochs   & 4                 \\
        norm\_adv        & True              \\
        clip\_coef       & 0.3               \\
        clip\_vloss      & True              \\
        ent\_coef        & 0.1               \\
        vf\_coef         & 1                 \\
        max\_grad\_norm  & 0.5               \\ \bottomrule
    \end{tabular}
\end{table}

For the neural DNF-MT actor, the hyperparameters of the auxiliary losses and
$\delta$ delay scheduling used are listed below:

\begin{table}[H]
    \begin{tabular}{lll} \toprule
        \begin{tabular}[c]{@{}l@{}}Hyperparameter\\Group\end{tabular}               & \begin{tabular}[c]{@{}l@{}}Hyperparameter\\ Name\end{tabular} & Value             \\ \midrule
        \multirow{3}{*}{\begin{tabular}[c]{@{}l@{}}Auxiliary\\ Loss\end{tabular}}   & dis\_weight\_reg\_lambda                                      & 0                 \\
                                                                                    & conj\_tanh\_out\_reg\_lambda                                  & 0                 \\
                                                                                    & mt\_lambda                                                    & $1\mathrm{e}{-3}$ \\ \hline
        \multirow{4}{*}{\begin{tabular}[c]{@{}l@{}}Delta\\ Scheduling\end{tabular}} & initial\_delta                                                & 0.1               \\
                                                                                    & delta\_decay\_delay                                           & 30                \\
                                                                                    & delta\_decay\_steps                                           & 5                 \\
                                                                                    & delta\_decay\_rate                                            & 1.1               \\ \bottomrule
    \end{tabular}
\end{table}

We provide more interpretability examples for Long Corridor-5 (LC-5) and Long
Corridor-11 (LC-11) environments in this section.

\noindent \textbf{LC-5.}\ \ The special state is also at state 1. Listing
\ref{code:lc5-sn-ppo-ndnf-mt-l4-1e5-aux-3112-asp} is the ASP program for a
neural DNF-MT actor in LC-5 MDP, with correctly identifying the special state,
state 1.

\begin{lstlisting}[
    label={code:lc5-sn-ppo-ndnf-mt-l4-1e5-aux-3112-asp},
    caption={ASP rules for a neural DNF-MT actor in LC-5 MDP.}
]
action(left) :- in_s_1.
action(right) :- not in_s_1.
\end{lstlisting}

Listing \ref{code:lc5-ws-ppo-ndnf-mt-l4-1e5-aux-8530-problog} is the ProbLog
rules for one neural DNF-MT actor of the runs in LC-5 POMDP. Again, with wall
status observations, the actor needs to be `flexible' when there is no wall on
either side, as shown in line 2 in Listing
\ref{code:lc5-ws-ppo-ndnf-mt-l4-1e5-aux-8530-problog}. The probability of going
left and right differs from the SC POMDP case, as there is one more state with
no walls on either side in LC5.

\begin{lstlisting}[
    label={code:lc5-ws-ppo-ndnf-mt-l4-1e5-aux-8530-problog},
    caption={ProbLog rules for a neural DNF-MT actor in LC-5 POMDP.}
]
0.139::action(left) ; 0.861::action(right) :- left_wall_present, \+ right_wall_present.
0.326::action(left) ; 0.674::action(right) :- \+ left_wall_present, \+ right_wall_present.
\end{lstlisting}

\noindent \textbf{LC-11.}\ \ The special states are 5, 6, 7 and 8, while the
agent needs to go from state 7 to state 3. The minimal path goes through $[7, 6,
        5, 4]$ in order, with action $[\text{right, right, right, left}]$.
        
Listing \ref{code:lc11-sn-ppo-ndnf-mt-l4-1e5-aux-2668-asp} is the ASP program
for a neural DNF-MT actor in LC-11 MDP. It correctly captures the special states
that it will go through in the minimal path in a single conjunction note
(`conj\_3' in line 3).

\begin{lstlisting}[
    label={code:lc11-sn-ppo-ndnf-mt-l4-1e5-aux-2668-asp},
    caption={ASP rules for a neural DNF-MT actor in LC-11 MDP.}
]
action(left) :- conj_3.
action(right) :- not conj_3.
conj_3 :- not in_s_5, not in_s_6, not in_s_7.
\end{lstlisting}

Listing \ref{code:lc11-ws-ppo-ndnf-mt-l4-1e5-aux-6707-problog} is the ProbLog
program for a neural DNF-MT actor in LC-11 POMDP. The left wall will never be
present since the agent will never see the state 0. This is reflected in the
ProbLog rules, where `left\_wall\_present' is never considered. The agent
prefers to go right when there is no right wall present. Under action sampling,
there is a chance of the agent getting to state 11, where the right wall is
present, as reflected in line 2. And we observe that the probability
distribution in state 11 is not 0\%-100\% but 75.6\%-24.4\% (line 2). This is
because the agent is not guaranteed to get to state 11 consistently, making the
training less `balanced' regarding state visiting frequency.

\begin{lstlisting}[
    label={code:lc11-ws-ppo-ndnf-mt-l4-1e5-aux-6707-problog},
    caption={ProbLog rules for a neural DNF-MT actor in LC-11 POMDP.}
]
0.244::action(left) ; 0.756::action(right) :-
    \+ right_wall_present.
0.756::action(left) ; 0.244::action(right) :-
    right_wall_present.
\end{lstlisting}

\subsection{Blackjack} \label{appendix:add-exp-blackjack}

The model architectures are listed below:

\begin{table}[H]
    \begin{tabular}{ll} \toprule
        Model         & Architecture                                                                                                                                                           \\ \midrule
        MLP actor     & \begin{tabular}[c]{@{}l@{}}(0): Linear(in=44, out=64, bias=True)\\(1): Tanh()\\(2): Linear(in=64, out=2, bias=True)\end{tabular}                                       \\ \hline
        NDNF-MT actor & \begin{tabular}[c]{@{}l@{}}(0): SemiSymbolic(\\\ \ \ \ \ \ in=44, out=64, type=CONJ.)\\(1): SemiSymbolicMutexTanh(\\\ \ \ \ \ \ in=64, out=2, type=DISJ.)\end{tabular} \\ \hline
        Critic        & \begin{tabular}[c]{@{}l@{}}(0): Linear(in=44, out=64, bias=True)\\(1): Tanh()\\(2): Linear(in=64, out=1, bias=True)\end{tabular}                                       \\ \hline
    \end{tabular}
\end{table}

The PPO hyperparameters used for training both the MLP actor and the neural
DNF-MT actor are listed below:

\begin{table}[H]
    \begin{tabular}{ll} \toprule
        Hyperparameter   & Value             \\ \midrule
        total\_timesteps & $3\mathrm{e}{5}$  \\
        learning\_rate   & $1\mathrm{e}{-3}$ \\
        num\_envs        & 32                \\
        num\_steps       & 16                \\
        anneal\_lr       & True              \\
        gamma            & 0.99              \\
        gae\_lambda      & 0.95              \\
        num\_minibatches & 16                \\
        update\_epochs   & 4                 \\
        norm\_adv        & True              \\
        clip\_coef       & 0.3               \\
        clip\_vloss      & True              \\
        ent\_coef        & 0.1               \\
        vf\_coef         & 1                 \\
        max\_grad\_norm  & 0.5               \\ \bottomrule
    \end{tabular}
\end{table}

For the neural DNF-MT actor, the hyperparameters of the auxiliary losses and
$\delta$ delay scheduling used are listed below:

\begin{table}[H]
    \begin{tabular}{lll} \toprule
        \begin{tabular}[c]{@{}l@{}}Hyperparameter\\Group\end{tabular}               & \begin{tabular}[c]{@{}l@{}}Hyperparameter\\ Name\end{tabular} & Value             \\ \midrule
        \multirow{3}{*}{\begin{tabular}[c]{@{}l@{}}Auxiliary\\ Loss\end{tabular}}   & dis\_weight\_reg\_lambda                                      & $1\mathrm{e}{-6}$ \\
                                                                                    & conj\_tanh\_out\_reg\_lambda                                  & 0                 \\
                                                                                    & mt\_lambda                                                    & $1\mathrm{e}{-3}$ \\ \hline
        \multirow{4}{*}{\begin{tabular}[c]{@{}l@{}}Delta\\ Scheduling\end{tabular}} & initial\_delta                                                & 0.1               \\
                                                                                    & delta\_decay\_delay                                           & 100               \\
                                                                                    & delta\_decay\_steps                                           & 10                \\
                                                                                    & delta\_decay\_rate                                            & 1.1               \\ \bottomrule
    \end{tabular}
\end{table}

Figure \ref{fig:blackjack-ndnf-mt-3191-soft-policy-cmp-q} shows the policy grids
of a neural DNF-MT actor trained in the Blackjack environment without any
post-training processing, and Figure
\ref{fig:blackjack-ndnf-mt-3191-problog-policy-cmp-q} shows its extracted
ProbLog policy grid. In both figures, a red square indicates that the argmax
action of the agent is different from the baseline Q-table from
\cite{rl-book-sutton-barto}. We see that the extracted ProbLog policy grid shows
more errors (red squares) compared to the original neural DNF-MT actor without
any post-training processing.

\begin{figure}[H]
    \centering
    \includegraphics[width=\columnwidth]{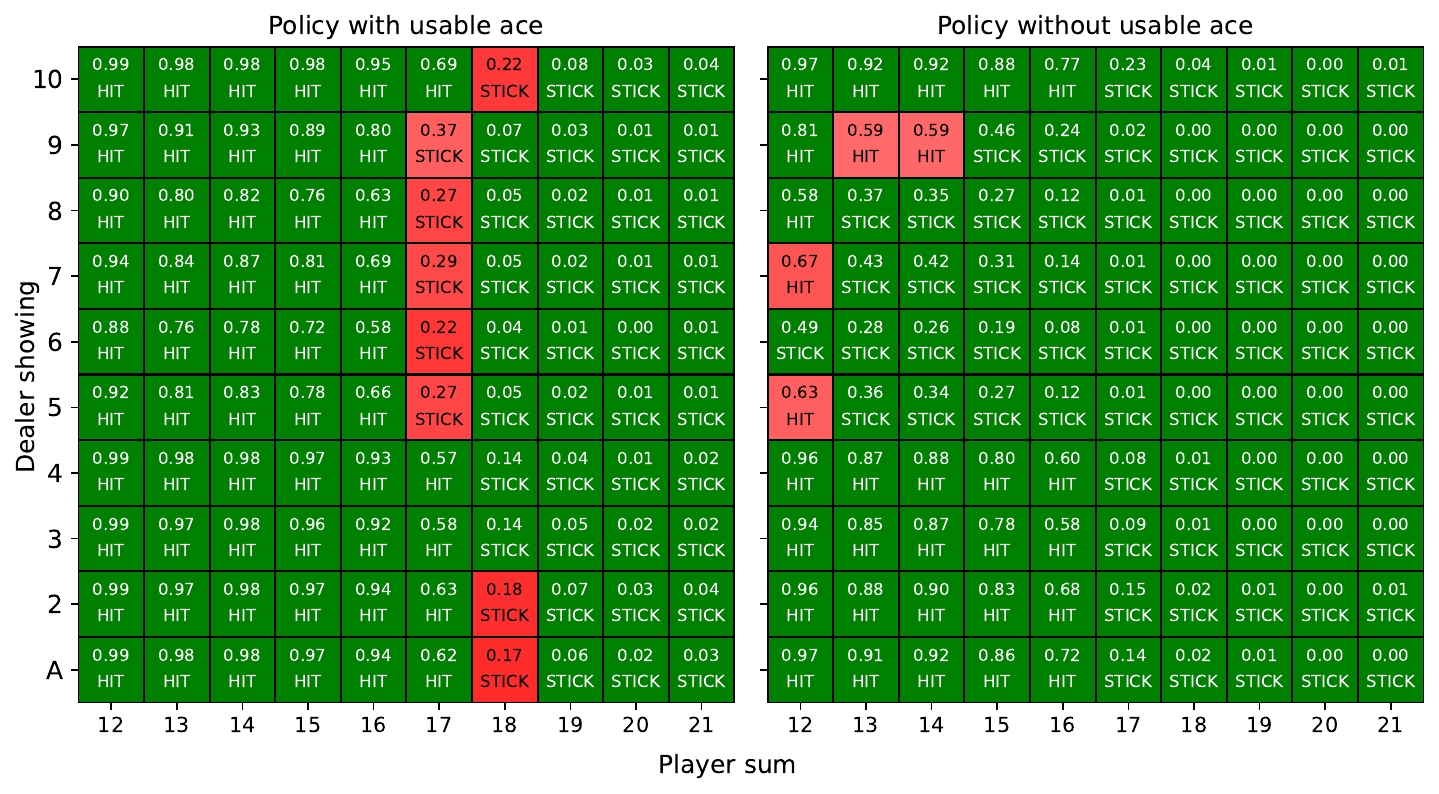}
    \caption{Policy grid of a neural DNF-MT actor in the Blackjack environment
        without post-training processing. Each square shows the probability of
        taking action `Hit' and the argmax action in the text. If a square is
        red, it means that its argmax action is different from the baseline
        Q-table from \cite{rl-book-sutton-barto}. The redder the square, the
        larger the probability difference.}
    \label{fig:blackjack-ndnf-mt-3191-soft-policy-cmp-q}
    \Description{Policy grid of a neural DNF-MT actor in the Blackjack
        environment without post-training processing.}
\end{figure}

\begin{figure}[H]
    \centering
    \includegraphics[width=\columnwidth]{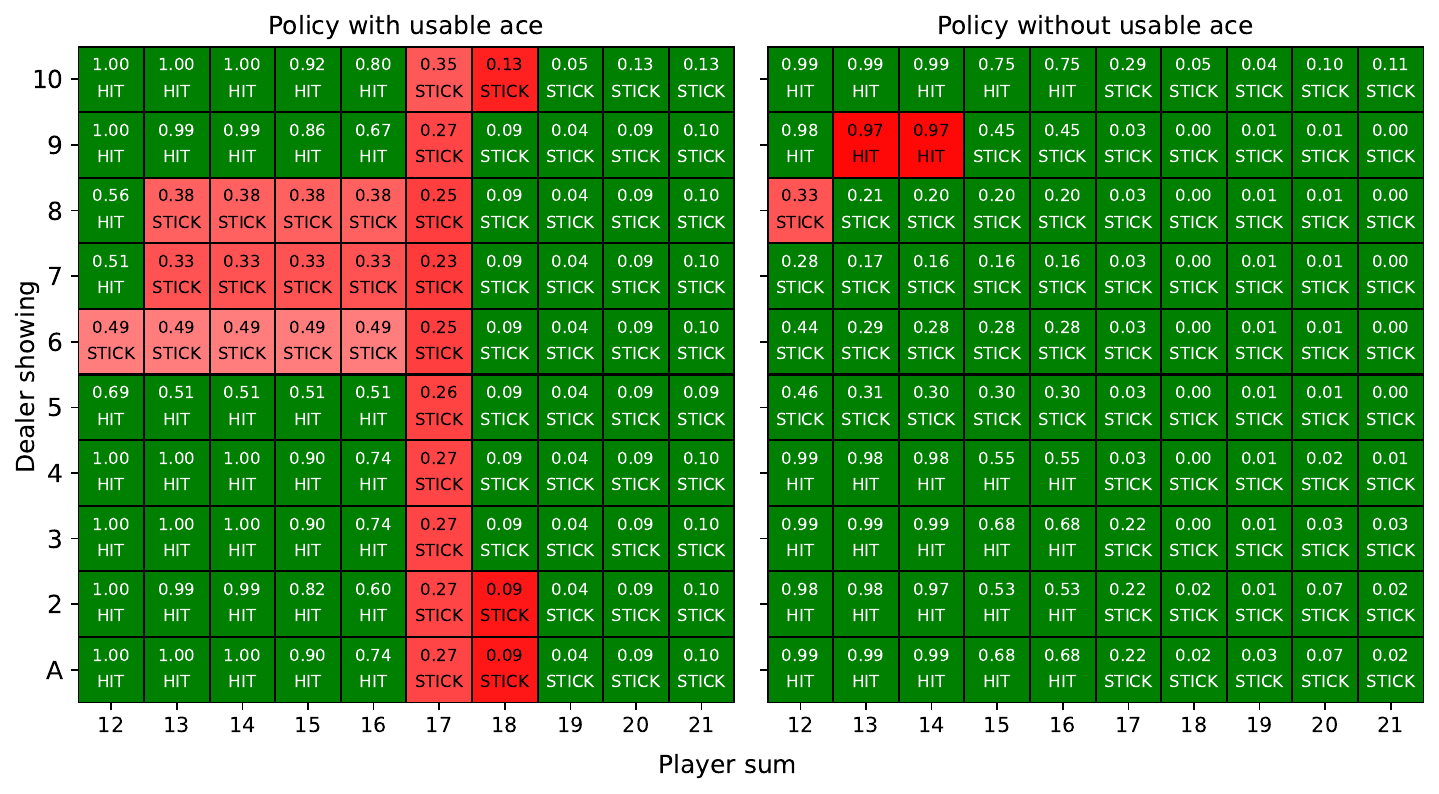}
    \caption{Extracted ProbLog policy grid of the same neural DNF-MT actor in
        Figure \ref{fig:blackjack-ndnf-mt-3191-soft-policy-cmp-q}.}
    \label{fig:blackjack-ndnf-mt-3191-problog-policy-cmp-q}
    \Description{Extracted ProbLog policy grid of the same neural DNF-MT actor
        in Figure \ref{fig:blackjack-ndnf-mt-3191-soft-policy-cmp-q}.}
\end{figure}

The neural DNF-MT actor's ProbLog policy corresponding to
Figure~\ref{fig:blackjack-ndnf-mt-3191-problog-policy-cmp-q} is shown in Listing
\ref{code:blackjack-ndnf-mt-3191-problog}. The entire program contains 37
non-probabilistic rules with `conj\_i' as the rule head, and 201 annotated
disjunctions with action probabilities.

\begin{lstlisting}[
    label={code:blackjack-ndnf-mt-3191-problog},
    caption={ProbLog rules for the neural DNF-MT actor as shown in Figure
        \ref{fig:blackjack-ndnf-mt-3191-problog-policy-cmp-q}, showing the 
        first five annotated disjunctions and the first five conjunction rules.}
]
0.873::action(stick) ; 0.127::action(hit) :- \+conj_3, \+conj_4, \+conj_5, \+conj_6, \+conj_7, \+conj_9, \+conj_10, \+conj_11, \+conj_12, \+conj_14, \+conj_15, \+conj_16, \+conj_19, \+conj_21, \+conj_22, \+conj_25, \+conj_29, \+conj_32, \+conj_34, \+conj_37, \+conj_40, \+conj_42, \+conj_44, \+conj_45, \+conj_47, \+conj_48, \+conj_49, \+conj_50, \+conj_51, \+conj_52, \+conj_53, \+conj_54, \+conj_55, \+conj_56, \+conj_57, \+conj_59, \+conj_62.
0.907::action(stick) ; 0.093::action(hit) :- \+conj_3, \+conj_4, \+conj_5, \+conj_6, \+conj_7, \+conj_9, \+conj_10, \+conj_11, \+conj_12, \+conj_14, \+conj_15, \+conj_16, \+conj_19, \+conj_21, \+conj_22, \+conj_25, \+conj_29, \+conj_32, \+conj_34, \+conj_37, \+conj_40, \+conj_42, \+conj_44, \+conj_45, \+conj_47, \+conj_48, \+conj_49, \+conj_50, \+conj_51, \+conj_52, \+conj_53, \+conj_54, conj_55, \+conj_56, \+conj_57, \+conj_59, \+conj_62.
0.900::action(stick) ; 0.100::action(hit) :- \+conj_3, \+conj_4, \+conj_5, \+conj_6, \+conj_7, \+conj_9, \+conj_10, \+conj_11, \+conj_12, \+conj_14, \+conj_15, \+conj_16, \+conj_19, \+conj_21, \+conj_22, \+conj_25, \+conj_29, \+conj_32, conj_34, \+conj_37, \+conj_40, \+conj_42, \+conj_44, \+conj_45, \+conj_47, \+conj_48, \+conj_49, \+conj_50, \+conj_51, \+conj_52, conj_53, \+conj_54, \+conj_55, \+conj_56, \+conj_57, \+conj_59, \+conj_62.
0.927::action(stick) ; 0.073::action(hit) :- \+conj_3, \+conj_4, \+conj_5, \+conj_6, \+conj_7, \+conj_9, \+conj_10, \+conj_11, \+conj_12, \+conj_14, \+conj_15, \+conj_16, \+conj_19, \+conj_21, \+conj_22, \+conj_25, \+conj_29, \+conj_32, conj_34, \+conj_37, \+conj_40, \+conj_42, \+conj_44, \+conj_45, \+conj_47, \+conj_48, \+conj_49, \+conj_50, \+conj_51, \+conj_52, conj_53, \+conj_54, conj_55, \+conj_56, \+conj_57, \+conj_59, \+conj_62.
0.983::action(stick) ; 0.017::action(hit) :- \+conj_3, \+conj_4, \+conj_5, \+conj_6, \+conj_7, \+conj_9, \+conj_10, \+conj_11, \+conj_12, \+conj_14, \+conj_15, \+conj_16, \+conj_19, \+conj_21, \+conj_22, \+conj_25, \+conj_29, \+conj_32, conj_34, \+conj_37, \+conj_40, \+conj_42, conj_44, \+conj_45, \+conj_47, \+conj_48, \+conj_49, \+conj_50, \+conj_51, \+conj_52, conj_53, \+conj_54, conj_55, \+conj_56, \+conj_57, \+conj_59, \+conj_62.
...
conj_3 :-
    hand(18), \+dealer(9), \+dealer(10), \+usable_ace.
conj_4 :-
    \+hand(17), \+hand(18), \+hand(19), \+hand(20), \+hand(21), \+hand(22), \+hand(28).
conj_5 :-
    hand(20), \+dealer(1), \+dealer(9), \+dealer(10), \+usable_ace.
conj_6 :-
    hand(12), \+dealer(5), usable_ace.
conj_7 :-
    hand(10), \+dealer(2), \+dealer(3), \+dealer(4), \+dealer(6).
...
\end{lstlisting}

\subsection{Taxi} \label{appendix:add-exp-taxi}

The model architectures are listed below:

\begin{table}[H]
    \begin{tabular}{ll} \toprule
        Model                                                                  & Architecture                                                                                                                                                                              \\ \midrule
        MLP actor                                                              & \begin{tabular}[c]{@{}l@{}}(0): Linear(in=500, out=256, bias=True)\\(1): Tanh()\\(2): Linear(in=256, out=6, bias=True)\end{tabular}                                                       \\ \hline
        Critic                                                                 & \begin{tabular}[c]{@{}l@{}}(0): Linear(in=500, out=256, bias=True)\\(1): ReLU()\\(2): Linear(in=256, out=256, bias=True)\\(3): ReLU()\\(4): Linear(in=256, out=1, bias=True)\end{tabular} \\ \hline
        \begin{tabular}[c]{@{}l@{}}MLP oracle\\(for distillation)\end{tabular} & \begin{tabular}[c]{@{}l@{}}(0): Linear(in=500, out=256, bias=False)\\(1): Tanh()\\(2): Linear(in=256, out=6, bias=False)\end{tabular}                                                     \\ \hline
        \begin{tabular}[c]{@{}l@{}}NDNF-MT actor\\(distilled)\end{tabular}     & \begin{tabular}[c]{@{}l@{}}(0): SemiSymbolic(\\\ \ \ \ \ \ in=500, out=64, type=CONJ.)\\(1): SemiSymbolicMutexTanh(\\\ \ \ \ \ \ in=64, out=6, type=DISJ.)\end{tabular}                   \\ \hline
    \end{tabular}
\end{table}

The PPO hyperparameters used for training the MLP actor are listed below:

\begin{table}[H]
    \begin{tabular}{ll} \toprule
        Hyperparameter         & Value             \\ \midrule
        total\_timesteps       & $3\mathrm{e}{6}$  \\
        learning\_rate\_actor  & $2\mathrm{e}{-4}$ \\
        learning\_rate\_critic & $2\mathrm{e}{-3}$ \\
        num\_envs              & 64                \\
        num\_steps             & 2048              \\
        anneal\_lr             & True              \\
        gamma                  & 0.999             \\
        gae\_lambda            & 0.946             \\
        num\_minibatches       & 128               \\
        update\_epochs         & 8                 \\
        norm\_adv              & True              \\
        clip\_coef             & 0.2               \\
        clip\_vloss            & True              \\
        ent\_coef              & 0.003             \\
        vf\_coef               & 0.5               \\
        max\_grad\_norm        & 0.5               \\ \bottomrule
    \end{tabular}
\end{table}

The distillation hyperparameters used for training neural DNF-MT actors are
listed below:

\begin{table}[H]
    \begin{tabular}{lll} \toprule
        \begin{tabular}[c]{@{}l@{}}Hyperparameter\\Group\end{tabular}               & \begin{tabular}[c]{@{}l@{}}Hyperparameter\\ Name\end{tabular} & Value             \\ \midrule
        \multirow{3}{*}{Distilaltion}                                               & batch\_size                                                   & 32                \\
                                                                                    & epoch                                                         & 5000              \\
                                                                                    & learning\_rate                                                & $1\mathrm{e}{-4}$ \\ \hline
        \multirow{3}{*}{\begin{tabular}[c]{@{}l@{}}Auxiliary\\ Loss\end{tabular}}   & dis\_weight\_reg\_lambda                                      & $1\mathrm{e}{-4}$ \\
                                                                                    & conj\_tanh\_out\_reg\_lambda                                  & $1\mathrm{e}{-5}$ \\
                                                                                    & mt\_lambda                                                    & $1\mathrm{e}{-4}$ \\ \hline
        \multirow{4}{*}{\begin{tabular}[c]{@{}l@{}}Delta\\ Scheduling\end{tabular}} & initial\_delta                                                & 0.1               \\
                                                                                    & delta\_decay\_delay                                           & 1000              \\
                                                                                    & delta\_decay\_steps                                           & 100               \\
                                                                                    & delta\_decay\_rate                                            & 1.1               \\ \bottomrule
    \end{tabular}
\end{table}

We provide an example of a neural DNF-MT actor's ProbLog rules in Listing
\ref{code:taxi-distill-ndnf-mt-309-problog}.

\begin{lstlisting}[
    label={code:taxi-distill-ndnf-mt-309-problog},
    caption={ProbLog rules for a distilled neural DNF-MT actor in the Taxi
    environment, showing five annotated disjunctions and first five conjunction rules.}
]
0.595::action(down) ; 0.000::action(up) ; 0.000::action(right) ; 0.405::action(left) ; 0.000::action(pickup) ; 0.000::action(dropoff) :- \+conj_0, \+conj_1, \+conj_3, \+conj_4, conj_5, conj_6, \+conj_7, \+conj_8, \+conj_9, \+conj_11, \+conj_12, \+conj_13, \+conj_14, conj_15, \+conj_16, \+conj_18, \+conj_19, conj_20, conj_21, \+conj_22, conj_23, conj_24, \+conj_25, \+conj_26, conj_27, \+conj_28, \+conj_29, conj_30, \+conj_31, \+conj_33, \+conj_34, \+conj_35, \+conj_37, \+conj_38, \+conj_39, conj_40, \+conj_41, \+conj_42, \+conj_43, \+conj_44, \+conj_45, \+conj_46, \+conj_47, \+conj_48, \+conj_50, \+conj_51, \+conj_52, \+conj_55, \+conj_56, conj_57, \+conj_58, conj_60, \+conj_62, conj_63.
0.602::action(down) ; 0.000::action(up) ; 0.000::action(right) ; 0.398::action(left) ; 0.000::action(pickup) ; 0.000::action(dropoff) :- \+conj_0, \+conj_1, \+conj_3, \+conj_4, conj_5, conj_6, \+conj_7, \+conj_8, \+conj_9, \+conj_11, conj_12, \+conj_13, \+conj_14, conj_15, \+conj_16, \+conj_18, \+conj_19, conj_20, conj_21, \+conj_22, conj_23, conj_24, \+conj_25, \+conj_26, conj_27, \+conj_28, \+conj_29, conj_30, \+conj_31, \+conj_33, \+conj_34, \+conj_35, \+conj_37, \+conj_38, \+conj_39, conj_40, \+conj_41, \+conj_42, \+conj_43, \+conj_44, \+conj_45, \+conj_46, \+conj_47, \+conj_48, \+conj_50, \+conj_51, \+conj_52, \+conj_55, \+conj_56, conj_57, \+conj_58, conj_60, \+conj_62, conj_63.
0.963::action(down) ; 0.000::action(up) ; 0.000::action(right) ; 0.036::action(left) ; 0.000::action(pickup) ; 0.001::action(dropoff) :- \+conj_0, \+conj_1, \+conj_3, \+conj_4, conj_5, conj_6, \+conj_7, \+conj_8, \+conj_9, \+conj_11, conj_12, \+conj_13, \+conj_14, conj_15, \+conj_16, \+conj_18, \+conj_19, conj_20, conj_21, \+conj_22, conj_23, conj_24, \+conj_25, \+conj_26, conj_27, \+conj_28, \+conj_29, conj_30, \+conj_31, \+conj_33, \+conj_34, \+conj_35, \+conj_37, \+conj_38, \+conj_39, conj_40, \+conj_41, \+conj_42, \+conj_43, \+conj_44, \+conj_45, \+conj_46, \+conj_47, conj_48, \+conj_50, \+conj_51, \+conj_52, \+conj_55, \+conj_56, \+conj_57, \+conj_58, conj_60, \+conj_62, conj_63.
0.000::action(down) ; 0.000::action(up) ; 0.000::action(right) ; 0.000::action(left) ; 1.000::action(pickup) ; 0.000::action(dropoff) :- conj_0, conj_1, conj_3, conj_4, conj_5, conj_6, \+conj_7, \+conj_8, \+conj_9, conj_11, conj_12, \+conj_13, \+conj_14, conj_15, \+conj_16, \+conj_18, conj_19, conj_20, conj_21, \+conj_22, conj_23, conj_24, \+conj_25, \+conj_26, conj_27, conj_28, \+conj_29, \+conj_30, \+conj_31, \+conj_33, \+conj_34, \+conj_35, \+conj_37, \+conj_38, conj_39, \+conj_40, \+conj_41, \+conj_42, \+conj_43, conj_44, \+conj_45, \+conj_46, \+conj_47, conj_48, \+conj_50, \+conj_51, conj_52, \+conj_55, \+conj_56, conj_57, \+conj_58, conj_60, \+conj_62, conj_63.
0.000::action(down) ; 0.000::action(up) ; 0.000::action(right) ; 0.000::action(left) ; 0.000::action(pickup) ; 1.000::action(dropoff) :- \+conj_0, conj_1, conj_3, \+conj_4, \+conj_5, conj_6, \+conj_7, \+conj_8, \+conj_9, conj_11, conj_12, \+conj_13, \+conj_14, \+conj_15, \+conj_16, \+conj_18, \+conj_19, conj_20, \+conj_21, \+conj_22, conj_23, \+conj_24, \+conj_25, \+conj_26, conj_27, \+conj_28, conj_29, \+conj_30, \+conj_31, \+conj_33, \+conj_34, \+conj_35, \+conj_37, \+conj_38, conj_39, conj_40, \+conj_41, \+conj_42, \+conj_43, conj_44, \+conj_45, \+conj_46, \+conj_47, conj_48, \+conj_50, \+conj_51, \+conj_52, \+conj_55, \+conj_56, \+conj_57, \+conj_58, conj_60, \+conj_62, conj_63.
...
conj_0 :- \+state(4), \+state(9), \+state(11), \+state(14), \+state(16), \+state(24), \+state(26), \+state(28), \+state(33), \+state(52), \+state(54), \+state(59), \+state(61), \+state(68), \+state(72), \+state(76), \+state(78), \+state(79), \+state(81), \+state(82), \+state(89), \+state(94), \+state(96), \+state(97), \+state(104), \+state(107), \+state(112), \+state(113), \+state(114), \+state(128), \+state(129), \+state(131), \+state(134), \+state(139), \+state(152), \+state(159), \+state(162), \+state(171), \+state(176), \+state(178), \+state(182), \+state(189), \+state(191), \+state(193), \+state(194), \+state(279), \+state(311), \+state(379), \+state(394), \+state(443), \+state(479), \+state(489).
conj_1 :- \+state(14), \+state(21), \+state(28), \+state(61), \+state(62), \+state(63), \+state(68), \+state(69), \+state(76), \+state(78), \+state(81), \+state(82), \+state(88), \+state(96), \+state(101), \+state(121), \+state(128), \+state(131), \+state(136), \+state(138), \+state(144), \+state(162), \+state(163), \+state(166), \+state(168), \+state(176), \+state(177), \+state(178), \+state(182), \+state(186), \+state(187), \+state(188), \+state(189), \+state(191), \+state(193), \+state(196), \+state(238), \+state(246), \+state(256), \+state(266), \+state(267), \+state(276), \+state(278), \+state(281), \+state(286), \+state(298), \+state(303), \+state(327), \+state(331), \+state(332), \+state(333), \+state(339), \+state(341), \+state(342), \+state(343), \+state(344), \+state(348), \+state(349), \+state(351), \+state(352), \+state(353), \+state(354), \+state(359), \+state(363), \+state(366), \+state(367), \+state(382), \+state(383), \+state(384), \+state(388), \+state(391), \+state(394), \+state(396), \+state(398), \+state(401), \+state(402), \+state(404), \+state(406), \+state(412), \+state(413), \+state(414), \+state(423), \+state(427), \+state(428), \+state(431), \+state(434), \+state(437), \+state(438), \+state(439), \+state(441), \+state(443), \+state(448), \+state(449), \+state(451), \+state(452), \+state(454), \+state(457), \+state(458), \+state(459), \+state(462), \+state(464), \+state(466), \+state(468), \+state(469), \+state(481), \+state(483), \+state(484), \+state(488), \+state(489), \+state(491), \+state(496), \+state(498).
conj_3 :- \+state(6), \+state(11), \+state(14), \+state(17), \+state(24), \+state(27), \+state(28), \+state(32), \+state(33), \+state(38), \+state(51), \+state(52), \+state(54), \+state(58), \+state(59), \+state(61), \+state(68), \+state(72), \+state(76), \+state(78), \+state(79), \+state(81), \+state(88), \+state(91), \+state(93), \+state(94), \+state(96), \+state(104), \+state(107), \+state(114), \+state(117), \+state(128), \+state(129), \+state(131), \+state(132), \+state(133), \+state(134), \+state(139), \+state(152), \+state(156), \+state(159), \+state(162), \+state(169), \+state(171), \+state(173), \+state(176), \+state(178), \+state(182), \+state(183), \+state(189), \+state(191), \+state(192), \+state(193), \+state(194), \+state(199), \+state(211), \+state(293), \+state(294), \+state(308), \+state(309), \+state(311), \+state(318), \+state(379), \+state(392), \+state(393), \+state(394), \+state(414), \+state(436), \+state(439), \+state(452).
conj_4 :- \+state(4), \+state(6), \+state(8), \+state(9), \+state(12), \+state(14), \+state(16), \+state(18), \+state(19), \+state(24), \+state(26), \+state(27), \+state(29), \+state(32), \+state(33), \+state(34), \+state(37), \+state(38), \+state(39), \+state(42), \+state(43), \+state(48), \+state(51), \+state(56), \+state(59), \+state(61), \+state(68), \+state(71), \+state(73), \+state(74), \+state(76), \+state(79), \+state(82), \+state(83), \+state(88), \+state(91), \+state(92), \+state(93), \+state(94), \+state(97), \+state(99), \+state(101), \+state(104), \+state(108), \+state(109), \+state(111), \+state(112), \+state(113), \+state(117), \+state(119), \+state(121), \+state(126), \+state(127), \+state(129), \+state(131), \+state(132), \+state(133), \+state(134), \+state(136), \+state(137), \+state(139), \+state(141), \+state(143), \+state(146), \+state(148), \+state(151), \+state(158), \+state(162), \+state(166), \+state(169), \+state(171), \+state(172), \+state(173), \+state(174), \+state(176), \+state(177), \+state(178), \+state(179), \+state(183), \+state(184), \+state(186), \+state(187), \+state(189), \+state(191), \+state(192), \+state(193), \+state(194), \+state(201), \+state(203), \+state(208), \+state(211), \+state(218), \+state(222), \+state(266), \+state(267), \+state(272), \+state(273), \+state(274), \+state(277), \+state(284), \+state(286), \+state(293), \+state(294), \+state(297), \+state(299), \+state(301), \+state(302), \+state(303), \+state(304), \+state(306), \+state(307), \+state(309), \+state(311), \+state(312), \+state(313), \+state(314), \+state(321), \+state(322), \+state(323), \+state(328), \+state(332), \+state(334), \+state(338), \+state(339), \+state(341), \+state(343), \+state(344), \+state(346), \+state(347), \+state(351), \+state(352), \+state(353), \+state(361), \+state(362), \+state(368), \+state(369), \+state(371), \+state(372), \+state(373), \+state(382), \+state(383), \+state(384), \+state(386), \+state(389), \+state(391), \+state(392), \+state(393), \+state(394), \+state(396), \+state(397), \+state(401), \+state(402), \+state(403), \+state(404), \+state(406), \+state(407), \+state(414), \+state(418), \+state(419), \+state(421), \+state(422), \+state(427), \+state(429), \+state(433), \+state(436), \+state(438), \+state(442), \+state(444), \+state(446), \+state(447), \+state(449), \+state(451), \+state(453), \+state(454), \+state(456), \+state(457), \+state(458), \+state(459), \+state(461), \+state(462), \+state(463), \+state(467), \+state(468), \+state(471), \+state(477), \+state(482), \+state(486), \+state(487), \+state(489), \+state(491), \+state(496), \+state(497), \+state(498).
conj_5 :- \+state(16), \+state(97), \+state(123), \+state(136), \+state(479).
...
\end{lstlisting}

\subsection{Door Corridor} \label{appendix:add-exp-dc}

The model architectures are listed below:

\begin{table}[H]
    \begin{tabular}{ll} \toprule
        Model                                                    & Architecture                                                                                                                                                          \\ \midrule
        \begin{tabular}[c]{@{}l@{}}Feature\\Encoder\end{tabular} & \begin{tabular}[c]{@{}l@{}}(0): Conv2d(2, 4, kernel\_size=(1, 1), stride=(1, 1))\\(1): Tanh()\\(2): Linear(in=36, out=16, bias=True)\end{tabular}                     \\ \hline
        MLP actor                                                & \begin{tabular}[c]{@{}l@{}}(0): Linear(in=16, out=64, bias=True)\\(1): Tanh()\\(2): Linear(in=64, out=4, bias=True)\end{tabular}                                      \\ \hline
        \begin{tabular}[c]{@{}l@{}}NDNF-MT\\Actor\end{tabular}   & \begin{tabular}[c]{@{}l@{}}(0): SemiSymbolic(\\\ \ \ \ \ \ in=16, out=12, type=CONJ.)\\(1): SemiSymbolicMutexTanh(\\\ \ \ \ \ \ in=12 out=4, type=DISJ.)\end{tabular} \\ \hline
        Critic                                                   & \begin{tabular}[c]{@{}l@{}}(0): Linear(in=16, out=64, bias=True)\\(1): Tanh()\\(2): Linear(in=64, out=1, bias=True)\end{tabular}                                      \\ \bottomrule
    \end{tabular}
\end{table}

The PPO hyperparameters used for training both the MLP actor and the neural
DNF-MT actor are listed below:

\begin{table}[H]
    \begin{tabular}{ll} \toprule
        Hyperparameter   & Value             \\ \midrule
        total\_timesteps & $3\mathrm{e}{5}$  \\
        learning\_rate   & $1\mathrm{e}{-2}$ \\
        num\_envs        & 8                 \\
        num\_steps       & 64                \\
        anneal\_lr       & True              \\
        gamma            & 0.99              \\
        gae\_lambda      & 0.95              \\
        num\_minibatches & 8                 \\
        update\_epochs   & 4                 \\
        norm\_adv        & True              \\
        clip\_coef       & 0.3               \\
        clip\_vloss      & True              \\
        ent\_coef        & 0.1               \\
        vf\_coef         & 1                 \\
        max\_grad\_norm  & 0.5               \\ \bottomrule
    \end{tabular}
\end{table}

For the neural DNF-MT actor, the hyperparameters of the auxiliary losses and $\delta$ delay scheduling
used are listed below:

\begin{table}[H]
    \begin{tabular}{lll} \toprule
        \begin{tabular}[c]{@{}l@{}}Hyperparameter\\Group\end{tabular}               & \begin{tabular}[c]{@{}l@{}}Hyperparameter\\ Name\end{tabular} & Value              \\ \midrule
        \multirow{4}{*}{\begin{tabular}[c]{@{}l@{}}Auxiliary\\ Loss\end{tabular}}   & dis\_weight\_reg\_lambda                                      & 0                  \\
                                                                                    & conj\_tanh\_out\_reg\_lambda                                  & 0                  \\
                                                                                    & mt\_lambda                                                    & $1\mathrm{e}{-3}$  \\
                                                                                    & embedding\_reg\_lambda                                        & $3\mathrm{e}{-15}$ \\ \hline
        \multirow{4}{*}{\begin{tabular}[c]{@{}l@{}}Delta\\ Scheduling\end{tabular}} & initial\_delta                                                & 0.1                \\
                                                                                    & delta\_decay\_delay                                           & 50                 \\
                                                                                    & delta\_decay\_steps                                           & 10                 \\
                                                                                    & delta\_decay\_rate                                            & 1.1                \\ \bottomrule
    \end{tabular}
\end{table}

The code repo provides a notebook on policy intervention in DC-T and DC-OT.
The notebook's path is \lstinline{notebooks/Door Corridor PPO NDNF-MT-6731 Intervention.ipynb}
(\href{https://anonymous.4open.science/r/ndnf_rl-B08D/notebooks/Door%20Corridor%20PPO%20NDNF-MT-6731%20Intervention.ipynb}{\textit{link to notebook}}).

  \section{Run Time Comparison} \label{appendix:add-disc-run-time}

Table \ref{tab:run-time-cmp} shows the run time comparison between different
models in different environments. All entries are run on a 10-core Apple M1 Pro
CPU with 16G RAM, and we use the ProbLog Python API to perform inference.

\begin{table}[H]
    \caption{Run time comparison between different models in different
        environments. The overall run time is the total time to run the model
        for the specified number of episodes. Models without * are run in a
        single environment in a loop for $n$ times. Models with * are run in 8
        parallel environments for a total of $n$ episodes.}
    \label{tab:run-time-cmp}
    \begin{tabular}{lllll}  \toprule
        Env.                                                                     & Model    & \begin{tabular}[c]{@{}l@{}}No. \\ episodes\end{tabular} & \begin{tabular}[c]{@{}l@{}}Overall\\ run\\ time (s)\end{tabular} & \begin{tabular}[c]{@{}l@{}}Avg.\\ run time (s)\\ per epidose\end{tabular} \\ \midrule
        \multirow{6}{*}{Blackjack}                                               & Q-table  & \multirow{5}{*}{$1\mathrm{e}{4}$}                       & 0.416                                                            & $4.156\mathrm{e}{-5}$                                                     \\
                                                                                 & MLP      &                                                         & 2.283                                                            & $2.283\mathrm{e}{-4}$                                                     \\
                                                                                 & MLP*     &                                                         & 1.157                                                            & $1.157\mathrm{e}{-4}$                                                     \\
                                                                                 & NDNF-MT  &                                                         & 5.564                                                            & $5.564\mathrm{e}{-4}$                                                     \\
                                                                                 & NDNF-MT* &                                                         & 1.574                                                            & $1.574\mathrm{e}{-4}$                                                     \\ \cline{3-5}
                                                                                 & ProbLog  & 10                                                      & 21.583                                                           & $2.158$                                                                   \\ \hline
        \multirow{6}{*}{Taxi}                                                    & Q-table  & \multirow{5}{*}{$1\mathrm{e}{4}$}                       & 1.202                                                            & $1.202\mathrm{e}{-4}$                                                     \\
                                                                                 & MLP      &                                                         & 17.338                                                           & $1.734\mathrm{e}{-3}$                                                     \\
                                                                                 & MLP*     &                                                         & 7.279                                                            & $7.279\mathrm{e}{-4}$                                                     \\
                                                                                 & NDNF-MT  &                                                         & 49.954                                                           & $4.995\mathrm{e}{-3}$                                                     \\
                                                                                 & NDNF-MT* &                                                         & 9.790                                                            & $9.790\mathrm{e}{-4}$                                                     \\ \cline{3-5}
                                                                                 & ProbLog  & 1                                                       & \multicolumn{2}{c}{Timed out after 30 min}                                                                                                   \\ \hline
        \multirow{5}{*}{\begin{tabular}[c]{@{}l@{}}Door\\ Corridor\end{tabular}} & MLP      & \multirow{4}{*}{$1\mathrm{e}{4}$}                       & 10.178                                                           & $1.018\mathrm{e}{-3}$                                                     \\
                                                                                 & MLP*     &                                                         & 0.073                                                            & $7.258\mathrm{e}{-6}$                                                     \\
                                                                                 & NDNF-MT  &                                                         & 36.726                                                           & $3.673\mathrm{e}{-3}$                                                     \\
                                                                                 & NDNF-MT* &                                                         & 0.111                                                            & $1.112\mathrm{e}{-5}$                                                     \\ \cline{3-5}
                                                                                 & ASP      & 1000                                                    & 25.296                                                           & $2.530\mathrm{e}{-2}$                                                     \\ \bottomrule
    \end{tabular}
\end{table}

We also run an experiment to approximate the time taken to perform inference in
ProbLog when the program is large. All programs are in the following form:

\balance

\begin{lstlisting}
p_0_1 :: action(0) ; ... ; p_0_n :: action(n) :- rule(0).
p_1_1 :: action(0) ; ... ; p_1_n :: action(n) :- rule(1).
...
p_m_1 :: action(0) ; ... ; p_m_n :: action(n) :- rule(m).
\end{lstlisting}

We try with $n \in \{2, 6\}, m \in \{1..18\}$. The time taken is shown in Figure
\ref{fig:time-problog-rules}, and we observe the overall trend of exponential
growth in time to the number of rules. This is only for approximation since the
actual policy does not resemble the format we use here. A Problog program
extracted from a neural DNF-MT actor in a Blackjack environment has 58 annotated
disjunctions. It takes 2652.212s to perform 380 times of inference, averaging 7s
per inference. Regardless, we can still conclude that inference with ProbLog is
more expensive than with neural models.

\begin{figure}[H]
    \centering
    \includegraphics[width=\columnwidth]{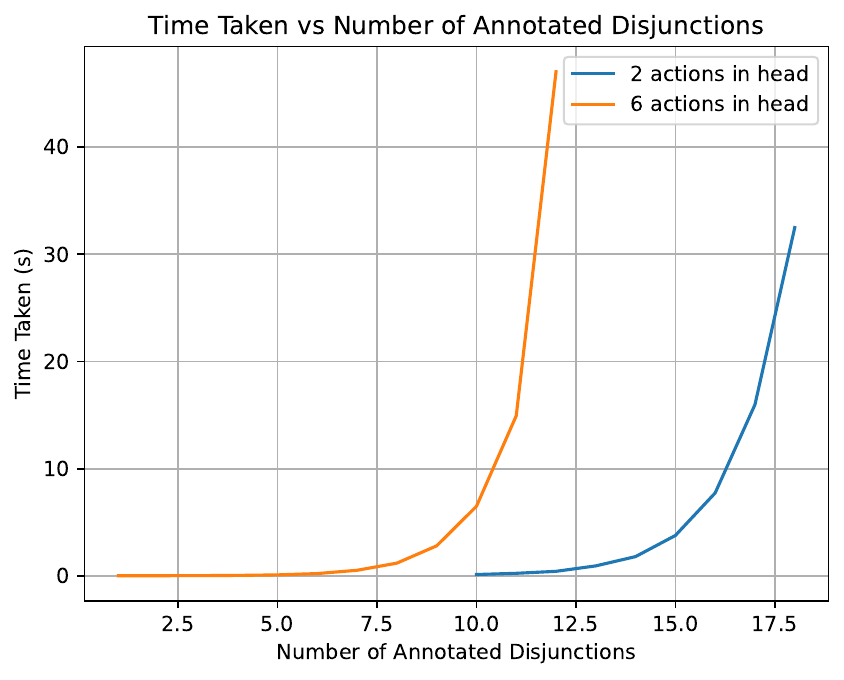}
    \caption{Time taken vs the number of annotated disjunctions in ProbLog,
        using a ProbLog program in a certain format as an approximation. The
        overall trend is that the time taken increases exponentially with the
        number of rules presented in the ProbLog program. However, we cannot
        confirm this pattern is the same as inference with an actual policy
        extracted from a neural DNF-MT actor.}
    \label{fig:time-problog-rules}
    \Description{Time taken vs number of annotated disjunctions in ProbLog. }
\end{figure}

\section{Performance Loss Due to Thresholding}

We discussed briefly in Section \ref{sec:discussions} on the performance loss
due to thresholding during the post-training processing. Here we show a further
insight on why this happens.

We go back to the derivation of the semi-symbolic node in pix2rule
\cite{pix2rule}. Say we have a conjunctive node with $N$ inputs, and each $x_i$
is strictly $\pm1$. The node is in a form of:

\begin{equation*}
    \sum_i^N w_i x_i + \beta = z
\end{equation*}

Recall the conjunction conditions: (1) if all atoms are true, the conjunction is
true; and (2) if any of the atoms is false, the conjunction is false. These are
formulated as:
\begin{equation} \label{eq:appendix-min-proof-1}
    \sum_i^N |w_i| + \beta = z
\end{equation}
\begin{equation} \label{eq:appendix-min-proof-2}
    \sum_{i, i \neq j}^N |w_i| - |w_j| + \beta = -z
\end{equation}
where $z$ is positive.

Combine Equation (\ref{eq:appendix-min-proof-1}) and
(\ref{eq:appendix-min-proof-2}):

\begin{align*}
    -2 \beta &= \sum_i^N |w_i| + \sum_{i, i \neq j}^N |w_i| - |w_j| \\
    -2 \beta &= 2 \sum_i^N |w_i| - 2 |w_j| \\
    \beta &= |w_j| - \sum_i^N |w_i|
\end{align*}

We use this as a starting point of determining what $\beta$ should be. Say that
there are more than one atoms in the conjunction being false, the conjunction
should remain false. So the output should be less than 0. Let $\beta =
\alpha - \sum_i^N |w_i|$, we calculate the output of the node:
\begin{align}
    \sum_{i \text{ s.t. } w_i x_i > 0}^N |w_i| - \sum_{j \text{ s.t. } w_j x_j < 0}^N |w_j| + \beta &< 0 \nonumber \\
    \sum_{i \text{ s.t. } w_i x_i > 0}^N |w_i| - \sum_{j \text{ s.t. } w_j x_j < 0}^N |w_j| + \alpha - \sum_i^N |w_i| &< 0 \nonumber \\
    \alpha - 2 \sum_{j \text{ s.t. } w_j x_j < 0} |w_j| &< 0 \nonumber\\
    \alpha < 2 \sum_{j \text{ s.t. } w_j x_j < 0} |w_j| \label{eq:appendix-min-proof-3}
\end{align}

Note that $\alpha - 2 \sum_{j \text{ s.t. } w_j x_j < 0} |w_j|$ can be see as
the output of a conjunctive node in general.

Going back to conjunction condition (2), the minimum case will be there is only
one atom being false in the conjunction. So Inequality
(\ref{eq:appendix-min-proof-3}) becomes:
\begin{equation}\label{eq:appendix-min-proof-4}
    \alpha < 2 |w_j|
\end{equation}
where $j$ is the index of the single atom being false.

So a suitable value of $\alpha$ for Equation (\ref{eq:appendix-min-proof-4})
would be $\min_{i, w_i \neq 0} |w_i|$. If $\min_{i, w_i \neq 0} |w_i|$ happen to
be the same value of $|w_j|$, we still have inequality $|w_j| < 2 |w_j|$. If
$w_j = 0$, then we go back to the case in Equation
(\ref{eq:appendix-min-proof-1}), and knowing the output of the layer is $\alpha
- 2 \sum_{j \text{ s.t. } w_j x_j < 0} |w_j|$, we know that the output would
still be greater than 0. If we take $\alpha = \min_{i, w_i \neq 0} |w_i|$ and
substitute it into Inequality (\ref{eq:appendix-min-proof-3}), the inequality
should always hold, as we have verified the base case of only one negation, and
adding more to the R.H.S would not change the result of the inequality. To
conclude, a conjunctive node's bias can be:
\begin{equation}
    \beta = \min_{i, w_i \neq 0} |w_i| - \sum_i^N |w_i|
\end{equation}

This is not the same as the $\max$ version of the bias proposed in pix2rule
\cite{pix2rule}, i.e. $\beta = \max_{i} |w_i| - \sum_i^N |w_i|$. The choice of
using max or min doesn't affect the logical semantics if all the weights are
strictly $\pm6$. However, in reality, we find that the weights are not all equal
but represent some form of `importance' of the input or how much the input
contributes to the belief, as shown in Section~\ref{sec:discussions}. Take the
following example where we have a conjunctive node with weights $[3, 1, 1]$ and
inputs $[1, -1, 1]$. The output of the node with max version of the bias is:
\begin{align*}
    &3 \cdot 1 + 1 \cdot (-1) + 1 \cdot 1 + \beta_{\max} \\
    =& 3 - 1 + 1 + (3 - (3 + 1 + 1)) \\
    =& 1 \nless 0
\end{align*}
Although we expect the node output to be false (i.e. $< 0$) as there is one atom
being false, the output layer is greater than 0 and interpreted as true. If we
use the min version of the bias, we have the expected behaviour:
\begin{align*}
    &3 \cdot 1 + 1 \cdot (-1) + 1 \cdot 1 + \beta_{\min} \\
    =& 3 - 1 + 1 + (1 - (3 + 1 + 1)) \\
    =& -1 < 0
\end{align*}

The max version of the bias is a more relaxed form of logic with the flexibility
to weigh inputs differently based on their `importance', and since the
thresholding does not consider this weighting, it might change the behaviour of
the node depends on the scale of the weights. On the other hand, the min version
of the bias is a stricter form of logic that is closer to bivalent logic, since
it does not consider the scale of the weights and thus the importance of inputs.
One may prefer the min version of the bias as it is stricter and closer to
bivalent logic. However, the stricter symbolic constraint also makes the
training harder. We find that the max version is easier to train and converge to
solutions in general. We leave the min version of the bias as a future work to
explore the trade-off between the two versions of the bias.

\end{appendix}


\end{document}